\documentclass[runningheads]{llncs}
\usepackage{graphicx}
\usepackage{comment}
\usepackage{amsmath,amssymb}
\usepackage{color}
\usepackage{epsfig}
\usepackage{xspace}
\usepackage{multirow}
\usepackage{booktabs}
\usepackage{comment}
\usepackage{textcomp}
\usepackage{subcaption}
\usepackage{tabulary}
\usepackage{epigraph}
\usepackage[table]{xcolor}
\definecolor{grey}{rgb}{0.9,0.9,0.9}
\usepackage{ctable}
\usepackage{algorithm}
\usepackage{algorithmic}
\usepackage{wrapfig}
\usepackage[pagebackref=true,breaklinks=true,colorlinks,bookmarks=false,hypertexnames=false]{hyperref}
\usepackage{breqn}
\usepackage{mathtools}
\usepackage{array,multirow,graphicx}
\usepackage{wrapfig}
\usepackage[normalem]{ulem}
\usepackage[accsupp]{axessibility}

\usepackage{amsmath,amsfonts,bm}

\def\eqref#1{equation~\ref{#1}}

\def\1{\bm{1}}

\newcommand{\test}{\mathcal{D_{\mathrm{test}}}}

\DeclareMathAlphabet{\mathsfit}{\encodingdefault}{\sfdefault}{m}{sl}
\SetMathAlphabet{\mathsfit}{bold}{\encodingdefault}{\sfdefault}{bx}{n}

\def\gA{{\mathcal{A}}}

\def\gM{{\mathcal{M}}}

\def\gP{{\mathcal{P}}}

\def\gS{{\mathcal{S}}}
\def\gT{{\mathcal{T}}}

\newcommand{\E}{\mathbb{E}}

\DeclareMathOperator*{\argmax}{arg\,max}
\DeclareMathOperator*{\argmin}{arg\,min}

\makeatletter
\DeclareRobustCommand\onedot{\futurelet\@let@token\@onedot}
\def\@onedot{\ifx\@let@token.\else.\null\fi\xspace}

\def\eg{\emph{e.g}\onedot} 
\def\ie{\emph{i.e}\onedot}

\makeatother

\def\model{TRAVL }

\begin{document}
\pagestyle{headings}
\mainmatter
\def\ECCVSubNumber{7315}

\title{Rethinking Closed-loop Training for Autonomous Driving}

\titlerunning{Rethinking Closed-loop Training for Autonomous Driving}

\author{Chris Zhang\textsuperscript{$\star$}\inst{1,2}\and
  Runsheng Guo\textsuperscript{$\star\dagger$}\inst{3}\and
  Wenyuan Zeng\textsuperscript{$\star$}\inst{1,2}\and\\
  Yuwen Xiong\inst{1,2}\and
  Binbin Dai\inst{1}\and
  Rui Hu\inst{1}\and
  Mengye Ren\textsuperscript{$\dagger$}\inst{4}\and
  Raquel Urtasun\inst{1,2}
}
\authorrunning{C. Zhang, et al.}
\institute{\mbox{Waabi
    \quad\and University of Toronto}
  \mbox{\and University of Waterloo
    \quad\and New York University}\\
  \email{\{czhang,wzeng,yxiong,bdai,rhu,urtasun\}@waabi.ai}\\
  \email{r9guo@uwaterloo.ca@uwaterloo.ca} \quad \email{mengye@nyu.edu}
}
\maketitle
\newcommand\blfootnote[1]{%
  \begingroup
  \renewcommand\thefootnote{}\footnote{#1}%
  \addtocounter{footnote}{-1}%
  \endgroup
}
\blfootnote{\textsuperscript{$\star$} Denotes equal contribution. \\\textsuperscript{$\dagger$} Work done during affiliation with Waabi.}
\newcommand{\norm}[1]{\left\lVert#1\right\rVert}

\begin{abstract}
  Recent advances in high-fidelity
  simulators~\cite{dosovitskiy2017carla,zhou2020smarts,manivasagam2020lidarsim}
  have enabled closed-loop training of autonomous driving agents, potentially
  solving the distribution shift in training v.s. deployment and allowing
  training to be scaled both safely and cheaply. However, there is a lack of
  understanding of how to build effective training benchmarks for closed-loop
  training. In this work, we present the first empirical study which analyzes
  the effects of different training benchmark designs on the success of learning
  agents, such as how to design traffic scenarios and scale training
  environments. Furthermore, we show that many popular RL algorithms cannot
  achieve satisfactory performance in the context of autonomous driving, as they
  lack long-term planning and take an extremely long time to train. To address
  these issues, we propose trajectory value learning (TRAVL), an RL-based
  driving agent that performs planning with multistep look-ahead and exploits
  cheaply generated imagined data for efficient learning. Our experiments show
  that \model can learn much faster and produce safer  maneuvers compared to all
  the baselines. For more information, visit the project website: 
  \href[]{https://waabi.ai/research/travl}{https://waabi.ai/research/travl}.
\end{abstract}

\section{Introduction}
Self-driving  vehicles require complex decision-making processes that guarantee
safety while maximizing comfort and progress towards the destination. Most
approaches have relied on hand-engineered planners that are built on top of
perception and motion forecasting modules. However, a robust decision process
has proven elusive, failing to handle  the complexity of the real world.

In recent years, several approaches have been proposed, aiming at exploiting
machine  learning to learn to drive. Supervised learning approaches such as
behavior cloning
\cite{muller2005off,bansal2018chauffeurnet,codevilla2018end,codevilla2019exploring,sauer2018conditional,prakash2021multi}
that learn from human demonstrations are amongst the most popular, as large
amounts of driving data can be readily obtained by instrumenting a vehicle to
collect data while a human is driving. However, learning in this open-loop
manner leads to distribution shift between training and
deployment~\cite{ross2011reduction,codevilla2019exploring,de2019causal}, as the
model does not understand the closed-loop effects of its actions when passively
learning from the expert data distribution.
\begin{figure}[t]
  \centering
  \includegraphics[width=\textwidth]{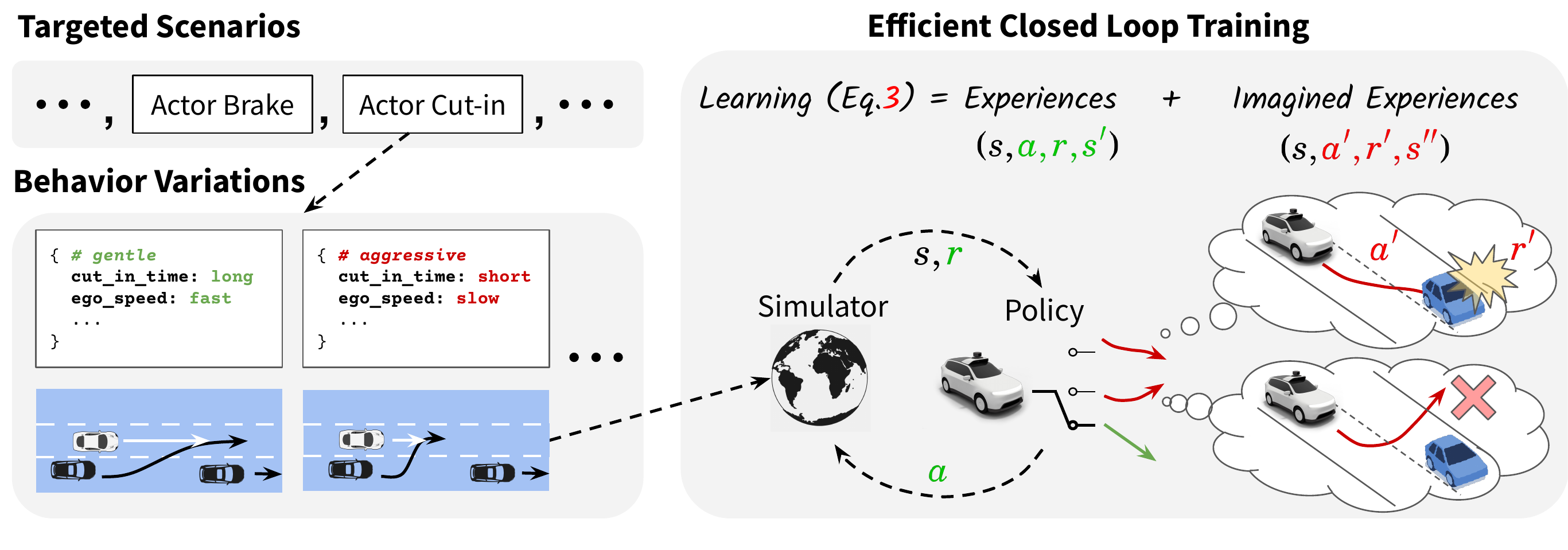}\hfill \caption{We use
    behavioral variations on top of scenarios designed to target specific
    traffic interactions as the basis for learning. \model efficiently learns to
    plan trajectories from both real and imagined experience.
    }\label{fig:teaser}
\end{figure}

Closed-loop training is one principled way to tackle this, by enabling the agent
to continuously interact with the environment and thus learn to recover from its
own mistakes. 
However, closing the loop while driving in the real world  comes
with many safety concerns as  it is dangerous to update the software on the fly
without proper safety verification. Furthermore, it is unethical to expose the
self-driving vehicle (SDV) to safety-critical situations, which is necessary for
learning to handle them. Finally, rare scenarios can take extremely long to
capture in the wild and is impractical to scale. An appealing alternative is to
learn to drive in a virtual environment by exploiting simulation
systems~\cite{dosovitskiy2017carla,zhou2020smarts}. While encouraging results
have been demonstrated
\cite{liang2018cirl,toromanoff2020end,sauer2018conditional}, there is still a
lack of understanding of how to build training benchmarks for effective
closed-loop training.\looseness=-1

In this paper we aim to shed some light on this by studying the following
questions: \emph{What type of scenarios do we need  to learn to drive safely?}
Simple scenarios used in previous works such as car
racing~\cite{wymann2000torcs,li2017infogail,pan2017agile} and
empty-lanes~\cite{kendall2019learning} are  insufficient in capturing the full
complexity of driving. Should we instead simulate complex free-flow
traffic\footnote{ This is similar to how we encounter random events when
collecting data. }\cite{halkias2006ngsim,henaff2019model} or design scenarios
targeting particular traffic situations \cite{dosovitskiy2017carla} that test
specific self-driving capabilities? We have seen in the context of  supervised
learning~\cite{deng2009imagenet,lin2014microsoft,houston2020one} that large
scale data improves generalization of learned models. Is this also the case in
closed-loop?
\emph{How many scenarios do we need?}
What effect does scaling our scenarios have on the quality of the learned
agents? How should we scale the dataset to best improve
performance?\looseness=-1

To better understand these questions, we present (to our knowledge) the first
study that analyzes the effect of different training benchmark designs on the
success of learning neural motion planners. Towards this goal, we developed a
sophisticated  highway driving simulator that can create both realistic
free-flow traffic as well as  targeted scenarios capable of  testing specific
self-driving capabilities. In particular, the latter is achieved by exploiting
procedural modeling which composes variations of unique traffic patterns (\eg,
lead actor breaking, merging in from an on ramp). Since each of these patterns
is parameterized, we can sample diverse variations and generate a large set of
scenarios automatically. Under this benchmark, we show:
\begin{enumerate}
  \item Scenarios designed for specific traffic interactions provide a richer
        learning signal than generic free-flow traffic simulations. This is
        likely because the former ensures the presence of interesting  and
        safety-critical interactions.
  \item
        Training on smaller scenario variations leads to more unsafe driving
        behaviors. This suggests that crafting more variations of traffic
        situations is key when building training benchmarks.
  \item
        Existing RL-based approaches have difficulty learning the intricacies of
        many scenarios. This is likely because they typically learn a direct
        mapping from observations to control signals (\eg, throttle, steering).
        Thus, they regrettably lack multi-step lookahead reasoning into the
        future, which is necessary to handle  complex scenarios such as merging
        into crowded lanes. Furthermore, learning these agents in a model-free
        manner can be extremely slow, as the agents have to learn with trial and
        error through costly simulations.
\end{enumerate}

To address the struggles of current RL approaches, we propose trajectory value
learning (TRAVL), a method which learns \emph{long-term reasoning efficiently in
closed-loop}. Instead of myopically outputting control commands independently at
each timestep, \model can perform decision-making with explicit multistep
look-ahead by \emph{planning} in trajectory space. Our model learns a deep
feature map representation of the state which can be fused with trajectory
features to directly predict the $Q$-value  of following that trajectory.
Inference amounts to selecting the maximum value trajectory plan from a sampled
trajectory set. Unlike conventional model-based planning, this bypasses the need
to explicitly model and predict all state variables and transitions, as not all
of them are equally important (\eg, a far away vehicle is of less interest in
driving). Furthermore, our trajectory-based formulation allows us to cheaply
produce additional \emph{imagined} (\ie, counterfactual) experiences, resulting
in  significantly better learning efficiency compared to model-free methods
which need to rely solely on interacting in the environment.

\subsubsection{Summary of contributions:}
In this paper, we present an in-depth empirical study on how various design
choices of training data generation can affect the success of learning driving
policies in closed-loop. This allows us to  identify a number of guidelines for
building effective closed-loop training benchmarks. We further propose a new
algorithm for efficient and effective learning of long horizon driving policies,
that better handle complex scenarios which mimic the complexity of the
real-world. We believe our work can serve as a starting point to rethink how we
shall conduct closed-loop training for autonomous driving.

\section{Related Work}
\subsubsection{Open-loop training:}
In open-loop training, the agent does not take any actions and instead learns
passively by observing expert states and actions. ALVINN
\cite{pomerleau1988alvinn} first explored behavior cloning as an open-loop
training method for end-to-end visuomotor self-driving. Since then, several
advances in data augmentation~\cite{bansal2018chauffeurnet,codevilla2018end},
neural network architecture
\cite{muller2005off,codevilla2018end,prakash2021multi} and auxiliary task design
\cite{codevilla2019exploring,sauer2018conditional} have been made in order to
handle more complex environments. Additionally, margin-based learning
approaches~\cite{zeng2019end,zeng2020dsdnet,sadat2019jointly,sadat2020perceive}
incorporate structured output spaces, while  offline
RL~\cite{shi2021offline,levine2020offline} exploits reward functions. The
primary challenge in open-loop training is the distribution shift encountered
when the predicted actions are rolled out in closed-loop. While techniques such
as cleverly augmenting the training data \cite{bansal2018chauffeurnet} partially
alleviate this issue, challenges remain.

\subsubsection{Closed-loop Training:}
The most popular paradigm for closed-loop training is reinforcement learning
(RL). In contrast to open-loop learning, online RL approaches
\cite{lillicrap2015continuous,schulman2015trust,schulman2017proximal}
do not require pre-collected expert data, but instead learn through interacting
with the environment. However,  such methods have prohibitively low sample
efficiency~\cite{chen2021learning} and can take several weeks to train a single
model~\cite{toromanoff2020end}. To address this issue, auxiliary tasks have been
used as additional sources of supervision, such as predicting affordances
\cite{chen2015deepdriving,sauer2018conditional,toromanoff2020end}, scene
reconstruction \cite{kendall2019learning} or imitation-based pre-training
\cite{liang2018cirl}. Note that our learning approach is orthogonal to these
tasks and thus can be easily combined. Model-based RL approaches are more sample
efficient. They assume access to a world model, which provides a cheaper way to
generate training data
\cite{sutton1990integrated,feinberg2018model,buckman2018sample}
in addition to the simulator. It can also be used during inference for planning
with multi-step lookaheads
\cite{nagabandi2018neural,hafner2019learning,chua2018deep,silver2018general,schrittwieser2020mastering,oh2017value,finn2017deep,srinivas2018universal}.
Yet, not many works have been explored in the context of self-driving.
\cite{chen2021learning} uses an on-rails world model to generate data for
training and empirically shows better efficiency. \cite{pan2019semantic}  uses a
learned semantic predictor to perform multistep look-ahead during inference. Our
work enjoys both benefits with a unified trajectory-based formulation.

When an expert can be queried online, other closed-loop training techniques such
as DAgger style
approaches~\cite{ross2011reduction,pan2017agile,chen2020learning} can be
applied. When the expert is only available during offline data collection, one
can apply closed-loop imitation learning
methods~\cite{ho2016generative,kuefler2017imitating}.  However, building an
expert can be expensive and difficult, limiting the applications of these
methods.

\subsubsection{Closed-loop Benchmarking:}
Directly closing the loop in the real world \cite{kendall2019learning} provides
the best realism but is unsafe. Thus, leveraging simulation has been a dominant
approach for autonomous driving. Environments focused on simple car
racing~\cite{wymann2000torcs,brockman2016openai} are useful in prototyping
quickly but can be over-simplified for real-world driving applications. More
complex traffic simulators
\cite{lopez2018microscopic,casas2010traffic,balmer2009matsim,ben2010traffic}
typically use heuristic-based actors to simulate general traffic flow.
Learning-based traffic models have also been explored
recently~\cite{bergamini2021simnet,suo2021trafficsim}. However, we  show that
while useful for evaluating the SDV in nominal conditions, general traffic flow
provides limited learning signal as interesting interactions are not guaranteed
to happen frequently. As the majority of traffic accidents can be categorized
into a few number of situations~\cite{najm2007pre}, recent works focus on
crafting traffic scenarios specifically targeting these situations for more
efficient coverage
\cite{shah2018airsim,apollo,dosovitskiy2017carla,ding2022survey,zhong2021survey,riedmaier2020survey}.
However, visual diversity is often more stressed than behavioral diversity. For
example, the CARLA benchmark \cite{dosovitskiy2017carla} has 10 different
scenario types and uses geolocation for variation, resulting in visually diverse
backgrounds but fixed actor policies
\footnote{\url{https://github.com/carla-simulator/scenario_runner}}. Other
approaches include mining real data~\cite{caesar2021nuplan}, or using
adversarial
approaches~\cite{wang2021advsim,koren2019efficient,ding2021multimodal,feng2021intelligent}.
Multi-agent approaches that control different actors with different
policies~\cite{zhou2020smarts,bernhard2020bark} have also been proposed.
However, these works have not studied the effects on learning in detail.

\section{Learning Neural Planners in Closed-loop}
\begin{figure}[t]
	\centering
	\includegraphics[width=\textwidth]{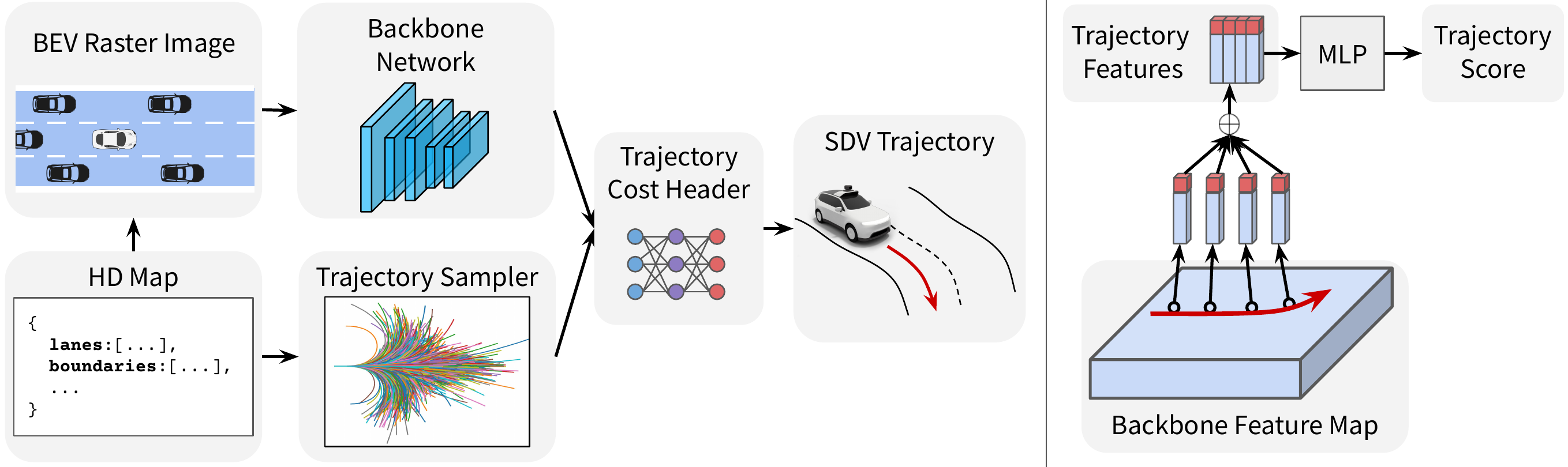}\hfill \caption{\model
		leverages rich backbone features to predict the cost of following a
		trajectory. The lowest costed trajectory is selected as the SDV's
		plan}\label{fig:architecture}
\end{figure}

Most model-free RL-based self-driving approaches parametrize the action space as
instant control signals (e.g., throttle, steering), which are directly predicted
from state observations. While simple, this parameterization hampers the ability
to perform long-term reasoning into the future, which is necessary in order to
handle complex driving situations. Furthermore, this approach can lead to
inefficient learning as  it relies solely on experiences collected from the
environment, which may contain only sparse supervisory signals. Model-based
approaches address these issues by explicitly  a  predictive model of the world.
However, performing explicit model rollouts online during inference can be
prohibitively expensive especially if the number of potential future
trajectories considered is very large.

We address these issues by combining aspects of model-free and model-based
approaches. In particular, we learn to reason into the future by directly
costing trajectories without explicit model rollouts, resulting in  more
efficient  inference. In addition to using real experience from the environment,
our trajectory output representation allows us to learn from imagined (\ie,
counterfactual) experiences collected from an approximate world model, greatly
improving sample efficiency. 

\subsection{Preliminaries on RL}
\label{sec:preliminaries}
The goal of an SDV is to make safe decisions sequentially. This can be modeled
as a Markov Decision Process (MDP) : $\gM = (\gS, \gA, P, R, \gamma)$, where
$\gS$ and $ \gA$ represent state and action spaces respectively, such as raw
observations of the scene and control signals for the ego-vehicle. $P(s'|s, a)$
and $R(s, a, s')$ represent the transition dynamics and reward functions
respectively, and $\gamma \in (0, 1)$ is the discount factor. We are  interested
in learning an optimal policy that maximizes the expected discounted return,
\[\pi^*(a | s) = \argmax_\pi \E_{\pi, P} \left[ \sum_{t=0}^T \gamma^t R(s^t,
		a^t, s^{t+1}) \right]\ \]

Off-policy RL algorithms are popular solutions due to their high data efficiency
since they are agnostic to the data collecting policy and thus do not constantly
require fresh data. A general form of the off-policy learning process can be
described as iteratively alternating between \textit{policy evaluation} and
\textit{policy improvement} steps. Specifically, one can maintain a learnable
$Q$-function $Q^k(s, a) = \E_{\pi, P} \left[ \sum_{t=0}^T \gamma^t R(s^t, a^t,
		s^{t+1}) \right]$, which captures the expected future return when executing
$\pi^k(a |s)$, with $k$  being the learning iteration. In the policy evaluation
step, the Bellman operator $\mathcal{B}_{\pi}Q \coloneqq R + \gamma \gP_{\pi}Q$
is applied to update $Q$ based on simulated data samples. Here, the transition
matrix $\gP_{\pi}$ is a matrix coupled  with the policy $\pi$, \ie,
$\gP_{\pi}Q(s,a) = \E_{s'\sim P(s'|s, a), a'\sim\pi(a'|s')}[Q(s', a')]$. We can
then improve the policy $\pi$ to favor selecting actions that maximize the
expected $Q$-value. However, both steps require evaluations on all possible $(s,
	a)$ pairs, and thus this is intractable for large state and action spaces. In
practice, one can instead apply empirical losses over a replay buffer, \eg,
minimizing the empirical $\ell_2$ loss between the left and right hand side of
the Bellman operator. The replay buffer is defined as the set
$\mathcal{D}=\{(s,a,r,s')\}$ holding past experiences sampled from $P(s'|s, a)$
and $\pi(a|s)$. Putting this together, the updating rules can be described as
follows,
\begin{align}
	\label{eq:actor-critic}
	 & Q^{k+1} \leftarrow \argmin_{Q} \E_{s, a, r, s' \sim \mathcal{D}} \left[
		\left(
		(r + \gamma \E_{a'\sim\pi^k}[Q^k(s', a')])
		-Q(s, a)
		\right)^2
		\right]
	\text{(evaluation)},\nonumber                                              \\
	 & \pi^{k+1} \leftarrow (1-\epsilon)\argmax_{\pi}\E_{s\sim\mathcal{D},
		a\sim\pi}[Q^{k+1}(s,a)] + \epsilon U(a),
	\quad\text{(improvement)}
\end{align}
where $U$ is the uniform distribution and $\epsilon$ is introduced for
epsilon-greedy exploration, \ie, making a greedy action under $Q$ with
probability $1-\epsilon$ and otherwise randomly exploring other actions with
probability $\epsilon$.
Note that Eq.~\ref{eq:actor-critic} reduces to standard $Q$-learning when  we
use $\pi^{k+1}(s) = \argmax_{a}Q^{k+1}(s, a)$ as the policy improvement step
instead.

\subsection{Planning with TRAVL}
Our goal is to design a driving model that can perform long term reasoning into
the future by \emph{planning}. To this end, we define our action as a
trajectory $\tau = \{(x^0, y^0), (x^1, y^1), \cdots, (x^T, y^T)\}$, which
navigates the ego-vehicle for the next $T$ timesteps. Here  $(x^t, y^t)$ is the
spatial location in birds eye view (BEV) at timestep $t$. Inspired by humans, we
decompose the cost of following a trajectory into a short-term
\emph{cost-to-come}, $R_{\theta}(s, \tau)$,  defined over the next $T$
timesteps,
and a long-term \emph{cost-to-go} $V_{\theta}(s,\tau)$ that operates beyond that
horizon.
The final $Q$-function is defined as   \[Q_{\theta}(s, \tau) = R_{\theta}(s,
	\tau) + V_{\theta}(s, \tau)\] Note that both $R_\theta$ and $V_\theta$ are
predicted with a neural network. In the following, we  describe our input state
representation,  the backbone network and cost predictive modules used to
predict $Q_\theta$ follow by our  planning inference procedure.

\paragraph{Input Representation:}
Following~\cite{bansal2018chauffeurnet,houston2020one,rhinehart2021contingencies},
our state space $\gS$ contains an HD map as well as the motion history of the
past $T'$ seconds of both the ego-vehicle and other actors.
To make the input data amenable to standard convolutional neural networks
(CNNs), we rasterize the information into a BEV tensor, where for each frame
within the history horizon $T'$, we draw bounding boxes of all actors as 1
channel using a binary mask. The ego-vehicle's past positions  are also
rasterized similarly into  $T'$ additional channels. We utilize an $M$ channel
tensor to represent the HD map, where each channel encodes a different map
primitive, such as centerlines or the target route. Finally, we include two more
channels to represent the $(x, y)$ coordinates of BEV
pixels~\cite{liu2018intriguing}. This results in a input tensor of  size
$\mathbb{R}^{H\times W\times (2T'+M+2)}$, where $H$ and $W$ denotes the size of
our input region around the SDV.

\paragraph{Backbone Network:}
To extract useful contextual information, we feed the input tensor to a backbone
network. As the input modality is a 2D image, we employ a CNN backbone adapted
from ResNet~\cite{he2016deep}. Given an input tensor of size
$\mathbb{R}^{H\times W\times (2T'+M+2)}$, the backbone performs downsampling and
computes a final feature map $\mathbf{F}\in \frac{H}{8} \times \frac{W}{8}
	\times C$, where $C$ is the feature dimension. More details on the architecture
are provided in the supplementary.

\paragraph{Cost Predictive Header:}
We use a cost predictive header that takes an arbitrary trajectory $\tau$ and
backbone feature map $\mathbf{F}$ as inputs, and outputs two scalar values
representing the cost-to-come and cost-to-go of executing $\tau$. As $\tau$ is
represented by a sequence of 2D waypoints $\{(x^0, y^0), (x^1, y^1), \cdots,
	(x^T, y^T)\}$, we can extract context features of $\tau$ by indexing the $t$
channel of the backbone feature $\mathbf{F}$ at position $(x^t, y^t)$ for each
timestep $t$. We then concatenate the features from all timesteps  into a single
feature vector $\mathbf{f}_{\tau}$. Note that the backbone feature $\mathbf{F}$
encodes rich information about the environment, and thus such an indexing
operation is expected to help reason about the goodness  of a trajectory, \eg,
if it is collision-free and follows the map. We also include kinematic
information of $\tau$ into $\mathbf{f}_{\tau}$ by concatenating the position
($x^t, y^t$), velocity ($v^t$), acceleration ($a^t$), orientation ($\theta_t$)
and curvature ($\kappa_t, \dot{\kappa}_t$).  We use two shallow (3 layer)
multi-layer perceptrons (MLPs) to regress the cost-to-come and cost-to-go from
$\mathbf{f}_{\tau}$ before finally summing them to obtain our estimate of $Q(s,
	\tau)$.

\paragraph{Efficient Inference:}
Finding the optimal policy $\tau^* \coloneqq \argmax_{\tau}Q(s, \tau)$ given a
$Q$-function over trajectories is difficult as $\tau$ lies in a continuous and
high-dimensional space that has complex structures (\eg, dynamic constraints).
To make inference efficient, we approximate such an optimization problem using a
sampling
strategy~\cite{schlechtriemen2016wiggling,werling2010optimal,sadat2019jointly,zeng2019end}.
Towards this goal, we first sample a wide variety of trajectories $\gT$ that are
physically feasible for the ego-vehicle, and then pick the one with maximum
$Q$-value
$$\tau^* = \argmax_{\tau\in\gT}Q(s, \tau).$$ To obtain a set of trajectory
candidates $\gT$, we use a map-based trajectory sampler, which samples a set of
lane following and lane changing trajectories following a bicycle
model~\cite{polack2017kinematic}. Inspired by \cite{sadat2019jointly}, our
sampling procedure is in the Frenet frame of the road, allowing us to easily
sample trajectories which consider map priors, \eg, follow curved lanes.
Specifically, longitudinal trajectories are obtained by fitting quartic splines
to knots corresponding to varying speed profiles, while lateral trajectories are
obtained by first sampling sets of various lateral offsets (defined with respect
to reference lanes) at different longitudinal locations and then fitting quintic
splines to them.  In practice, we find embedding map priors in this manner can
greatly improve the model performance. In our experiments we sample roughly 10k
trajectories per state. Note that despite outputting an entire trajectory as the
action, for inference we  use an MPC style execution~\cite{camacho2013model}
where the agent only executes an initial segment of the trajectory before
replanning with the latest observation. Further discussion on this method of
planning can be found in the supplementary.

\subsection{Efficient Learning with Counterfactual Rollouts}
\label{sec:learning}
The most straightforward way to learn our model is through classical RL
algorithms. We can  write the policy evaluation step in
Eq.~\ref{eq:actor-critic} as
\begin{align}
	Q^{k+1} \leftarrow \argmin_{Q_{\theta}} \mathbb{E}_{\mathcal{D}}\left[
		\left(Q_{\theta} - \mathcal{B}_{\pi}^k Q^k\right)^2
		\right], \quad s.t. \quad Q_{\theta} = R_{\theta} + V_{\theta},
\end{align}
where $\mathcal{B}_{\pi}Q \coloneqq R + \gamma \gP_{\pi}Q$ is the Bellman
operator. However, as we show in our experiments and also demonstrated in other
works~\cite{wang2019benchmarking}, such model-free RL algorithms learn very
slowly. Fortunately, our trajectory-based formulation and decomposition of
$Q_{\theta}$ into $R_{\theta}$ and $V_{\theta}$ allow us to design a  more
efficient  learning algorithm that follows the spirit of model-based approaches.

One of the  main benefits of model-based RL~\cite{sutton1990integrated} is the
ability to efficiently sample imagined data through the  world dynamics model,
as this  helps bypass the need for unrolling the policy in simulation  which can
be computationally expensive. However, for complex systems the learned model is
likely to be also expensive, \eg, neural networks. Therefore, we propose a
simple yet effective world model where we assume the actors are not intelligent
and do not react to different SDV trajectories. Suppose we have a replay buffer
$\mathcal{D} = \{(s^t, \tau^t, r^t, s^{t+1})\}$ collected by interacting our
current policy $\pi$ with the simulator. To augment the training data with
cheaply generated imagined data $(s^{t}, \tau', r', s')$, we consider a
counterfactual trajectory $\tau'$ that is different from $\tau$. The resulting
counterfactual state $s'$ simply modifies $s^{t+1}$ such that the ego-vehicle
follows $\tau'$, while keeping the actors' original states the same. The
counterfactual reward can then be computed as $r' = R(s^t, \tau', s')$. We
exploit our near limitless source of counterfactual data for dense supervision
on the short term predicted cost-to-come $R_\theta$. Efficiently learning
$R_{\theta}$ in turn benefits the learning of the $Q$-function overall. Our
final learning objective (policy evaluation) is then
\begin{align}
	\label{eq:objective}
	Q^{k+1} & \leftarrow \argmin_{Q_{\theta}} \mathbb{E}_{\mathcal{D}}\left[
		\underbrace{\left(Q_{\theta}(s, \tau) - \mathcal{B}_{\pi}^k Q^k(s, \tau)\right)^2}_{\text{Q-learning}} +
		\alpha_k \underbrace{\mathbb{E}_{\tau' \sim \mu(\tau'|s)}\left(R_{\theta}(s,
			\tau') - r'\right)^2 }_{\text{Counterfactual Reward Loss}}
	\right],\nonumber                                                                    \\
	        & s.t. \quad Q_{\theta} = R_{\theta} + V_{\theta}, \quad V_{\theta} = \gamma
	\gP_{\pi}^kQ^k.
\end{align}
Here, $\mu$ is an arbitrary distribution over possible trajectories and
characterizes the exploration strategy for counterfactual supervision. We use a
uniform distribution over the sampled trajectory set $\mathcal{T}$ in all our
experiments for simplicity. To perform an updating step in
Eq.~\ref{eq:objective}, we approximate the optimal value of $Q_{\theta}$ by
applying a gradient step of the empirical loss of the objective. In practice,
the gradients are applied over the parameters of $R_{\theta}$ and $V_{\theta}$,
whereas $Q_{\theta}$ is simply obtained by $R_{\theta} + V_{\theta}$. We also
find that simply using the $Q$-learning policy improvement step suffices in
practice. We provide more implementation details in the supplementary material.
Note that we use counterfactual rollouts, and thus the simplified non-reactive
world dynamics, only as supervision for short term reward/cost-to-come
component. This can help avoid compounding errors of using such an assumption
for long-term imagined simulations.

\paragraph{Theoretical analysis:}
The counterfactual reward loss component can introduce modeling errors due to
our approximated world modeling. As the $Q$-learning loss term in
Eq.~\ref{eq:objective} is decreasing when $k$ is large, such errors can have
significant effects, hence it is non-obvious whether our learning procedure
Eq.~\ref{eq:objective} can satisfactorily converge. We now show that iteratively
applying policy evaluation in Eq.~\ref{eq:objective} and policy update in
Eq.~\ref{eq:actor-critic} indeed converges under some mild conditions.

\begin{lemma}
	Assuming $R$ is bounded by a constant $R_{\text{max}}$ and $\alpha_k$
	satisfies
	\begin{equation}
		\alpha_k <
		\left(\frac{1}{\gamma^kC}
		-1\right)^{-1}\left(\frac{\pi^k}{\mu}\right)_{min},
	\end{equation}
	with $C$ an arbitrary constant,  iteratively applying Eq.~\ref{eq:objective}
	and the policy update step in Eq.~\ref{eq:actor-critic} converges to a fixed
	point.
\end{lemma}
Furthermore,  it converges to the optimal $Q$-function, $Q^*$. We refer the
reader to the supplementary material for detailed  proofs.
\begin{theorem}
	Under the same conditions as Lemma 1, our learning procedure converges to
	$Q^*$.
\end{theorem}

\section{Large Scale Closed-loop Benchmark Dataset}
\label{sec:benchmark}
\begin{figure}[t]
  \centering
  \includegraphics[width=\textwidth]{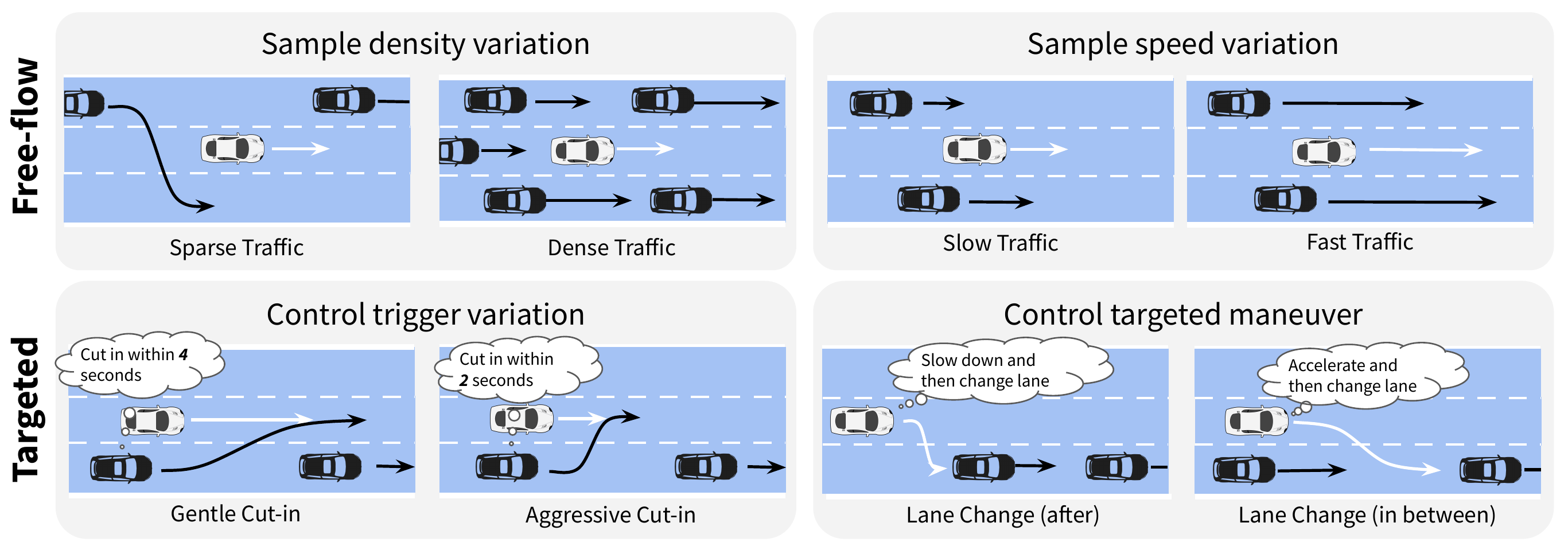}\hfill
  \caption{\textbf{Top row:} Free-flow scenarios  generated by sampling
    parameters such as density and actor speed. \textbf{Bottom row:} Targeted
    scenarios  generated by enacting  fine-grained control on the actors to
    target specific traffic situations. }\label{fig:targeted-freeflow-comp}
\end{figure}

We now describe how we design our large scale closed-loop benchmark. In our
simulator
the agent drives on a diverse set of highways, which contain standard, on-ramp,
merge and fork map topologies with varying curvature and number of lanes. We use
IDM~\cite{treiber2000congested} and MOBIL~\cite{kesting2007general} policies to
control actors. As our work is focused on learning neural planners in
closed-loop, we simulate bounding boxes and trajectory information of actors,
and leave sensor simulation for future work.

There are two popular ways to testing an autonomous driving system: 1)
uncontrolled traffic environments and 2) controlled scenarios that test certain
self-driving capabilities such as reacting to a cut-in. However, there has been
no analysis in the literature of the effects of \emph{training} in these
different environments. Towards this goal, we construct two training benchmarks
based on these two different paradigms.

\subsubsection{Free-flow Scenario Set:}
Free-flow scenarios are similar to what we observe in real-world  data, where we
do not enact any fine-grained control over other actors. We define a generative
model which samples from a set of parameters which define a scenario. Parameters
which vary the initial conditions of actors include density, initial speed,
actor class, and target lane goals. Parameters which vary actor driving behavior
include actor target speed, target gap, speed limit, maximum acceleration and
maximum deceleration. We also vary the map topology, road curvature, geometry,
and the number of lanes. We use a mixture of truncated normal distributions for
continuous properties and categorical distribution for discrete properties. More
details are provided in the supplementary.

\begin{table*}[t]
  \centering
  \setlength{\tabcolsep}{3pt}
  \begin{tabular}{c|c|ccccc}
    \specialrule{.2em}{.1em}{.1em}
    \multicolumn{2}{c|}{Method}                                           & Pass Rate $\uparrow$ & Col. Rate $\downarrow$
                                                                          & Prog. $\uparrow$     & MinTTC$\uparrow$       & MinDist$\uparrow$        \\
    \hline
    Imit. Learning                                                        & \multirow{4}{*}{C}   & 0.545                  & 0.177
                                                                          & 240                  & 0.00                   & 2.82                     \\
    PPO~\cite{schulman2017proximal}                                       &                      & 0.173                  & 0.163
                                                                          & 114                  & 0.00                   & 5.56                     \\
    A3C~\cite{mnih2016a3c}                                                &                      & 0.224                  & 0.159
                                                                          & 284                  & 0.03                   & 4.65                     \\
    RAINBOW\textsuperscript{\ref{simp-rainbow}} ~\cite{hessel2018rainbow} &
                                                                          & 0.435                & 0.270                  & 234               & 0.00
                                                                          & 1.38                                                                     \\
    \hline
    Imit. Learning                                                        & \multirow{5}{*}{T}   & 0.617                  & 0.261
                                                                          & \textbf{286}         & 0.00                   & 1.49                     \\
    PPO~\cite{schulman2017proximal}                                       &                      & 0.273                  & 0.249
                                                                          & 200                  & 0.00                   & 1.73                     \\
    A3C~\cite{mnih2016a3c}                                                &                      & 0.362                  & 0.137
                                                                          & 135                  & 0.30                   & 6.14                     \\
    RAINBOW\textsuperscript{\ref{simp-rainbow}} ~\cite{hessel2018rainbow} &
                                                                          & 0.814                & 0.048                  & 224               & 0.45
                                                                          & 9.70                                                                     \\
    \rowcolor{grey} \model (ours)                                         &                      & \textbf{0.865}         & \textbf{0.026}
                                                                          & 230                  & \textbf{0.82}          & \textbf{12.62}           \\
    \specialrule{.1em}{.05em}{.05em}
  \end{tabular}
  \caption{We compare our approach against several baselines. Here C is using
    the standard control setting and T is using our proposed trajectory-based
    architecture and formulation. We see that trajectory-based approaches
    outperforms their control-based counterparts. We also see that our proposed
    method is able to learn more efficiently and outperform baselines. }
  \label{table:sota}
\end{table*}

\subsubsection{Targeted Scenario Set:}
Targeted scenarios are those which are designed to test autonomous vehicles in
specific traffic situations. These scenarios are designed to ensure an
autonomous vehicle has certain capabilities or meets certain requirements (\eg,
the ability to stop for a leading vehicle braking, the ability to merge onto the
highway). In this work, we identified 3 ego-routing intentions (lane follow,
lane change, lane merge) and 5 behavior patterns for other agents (braking,
accelerating, blocking, cut-in, negotiating). Combining these options along with
varying the number of actors and where actors are placed relative to the ego
gives us a total of 24 scenario types (\eg lane change with leading and trailing
actor on the target lane). Each scenario type is then parameterized by a set of
scenario-specific parameters such as  heading and speed of the ego at
initialization, the relative speed and location of other actors at
initialization, time-to-collision and distance thresholds for triggering
reactive actions (\eg when an actor performs a cut-in), IDM parameters of other
actors as well as  map parameters. We then procedurally generate variations of
these scenarios by varying the values of these parameters, which result in
diverse scenario realizations with actor behaviors that share similar semantics.

\subsubsection{Benchmarking:}
We briefly explain how we sample parameters for scenarios. As the free-flow
scenarios aim to capture nominal traffic situations, we simply sample
\emph{i.i.d} random parameter values and hold out a test set. In contrast, each
targeted scenario serves for benchmarking a specific driving capability, and
thus we should prevent training and testing on the same (or similar) scenarios.
To this end, we first generate a test set of scenarios aiming to provide
thorough evaluations over the entire parameterized spaces. Because enumerating
all possible combinations of parameter is intractable, we employ an all-pairs
generative approach~\cite{nie2011survey} which provides a much smaller set that
contains all possible combinations for any pair of discrete parameters. This is
expected to provide efficient testing while covering a significant amount of
failure cases~\cite{mira2019test}. More details on this approach are in the
supplementary. Finally, we hold out those test parameters when drawing random
samples for the training and validation set.

\section{Experiments}
\label{experiments}
\begin{table*}[t]
  \centering
  \setlength{\tabcolsep}{3pt}
  \begin{tabular}{c|cc|cc|cc|cc}
    \specialrule{.2em}{.1em}{.1em}
    \multicolumn{3}{c|}{}                                             &
    \multicolumn{6}{c}{Test}                                                                                                               \\
    \cline{4-9} \multicolumn{3}{c|}{}                                 &
    \multicolumn{2}{c|}{Pass Rate $\uparrow$}                         &
    \multicolumn{2}{c|}{Collision Rate $\downarrow$}                  &
    \multicolumn{2}{c}{Progress $\uparrow$}                                                                                                \\
    \multicolumn{1}{c}{}                                              &
    \multicolumn{2}{c|}{}
                                                                      & Free-flow & Targeted & Free-flow & Targeted & Free-flow & Targeted \\ \hline
    \parbox[t]{2mm}{\multirow{4}{*}{\rotatebox[origin=c]{90}{Train}}} &
    \multirow{2}{*}{RB\textsuperscript{\ref{simp-rainbow}}+T}         &
    Free-flow                                                         & 0.783     & 0.453    & 0.198     & 0.228    & 146       & 173      \\ & &
    Targeted                                                          & 0.885     & 0.815    & 0.104     & 0.048    & 231       & 224      \\
    \cline{2-9} \noalign{\vskip\doublerulesep \vskip-\arrayrulewidth}\cline{2-9}
                                                                      &
    \multirow{2}{*}{TRAVL}                                            &
    Free-flow                                                         & 0.784     & 0.696    & 0.198     & 0.177    & 229       & 219      \\ & &
    Targeted                                                          & 0.903     & 0.865    & 0.089     & 0.026    & 172       & 230      \\
    \specialrule{.1em}{.05em}{.05em}
  \end{tabular}
  \caption{We train RAINBOW\textsuperscript{\ref{simp-rainbow}} and TRAVL on
    different sets and evaluate on different sets. We see that training on
    targeted scenarios performs better than training on free-flow scenarios,
    even when evaluated on free-flow scenarios. }
  \label{table:target_set}
\end{table*}

In this section, we showcase the benefits of our proposed learning method \model
by comparing against several baselines. We empirically study the importance of
using targeted scenarios compared to free-flow scenarios for training and
evaluation. Finally, we further study the effect of data diversity and scale and
show that large scale, behaviorally diverse data is crucial in learning good
policies.

\subsubsection{Datasets and Metrics:}
The free-flow dataset contains 834 training and 274 testing scenarios. The
targeted dataset contains 783 training and 256 testing scenarios. All scenarios
last for 15 seconds on average.

We use a number of autonomy metrics for evaluating safety and progress.
\emph{Scenario Pass Rate} is the percentage of scenarios that pass,  which is
defined as reaching the goal (e.g., maintain a target lane, reach a distance in
a given time) without collision or speeding violations. \emph{Collision Rate}
computes the percentage of scenarios where ego collides. \emph{Progress}
measures the distance traveled in meters before the scenario ends.
\emph{Minimum Time to Collision (MinTTC)}
measures the time to collision with another actor if the ego state were
extrapolated along its future executed trajectory, with lower values meaning
closer to collision.
\emph{Minimum Distance to the Closest Actor (MinDistClAct)}
computes the minimum distance between the ego and any other actor in the scene.
We use the median when reporting metrics as they are less sensitive to outliers.

\subsubsection{Closed-loop Benchmarking:}
We train and evaluate \model and several baselines on our proposed targeted
scenario set. We evaluate A3C~\cite{mnih2016a3c},
PPO~\cite{schulman2017proximal} and a simplified RAINBOW\footnote{We only
  include prioritized replay and multistep learning as they were found to be the
  most applicable and important in our setting. \label{simp-rainbow}}
~\cite{hessel2018rainbow} as baseline RL algorithms, and experiment with control
(C) and trajectory (T) output representations. We also include imitation
learning (IL) baselines supervised from a well-tuned auto-pilot. In the control
setting, the action space consists of 9 possible values for steering angle and 7
values for throttle, yielding a total of 63 discrete actions. In the trajectory
setting, each trajectory sample is treated as a discrete action. All methods use
the same backbone and header architecture, trajectory sampler, and MPC style
inference when applicable. Our reward function consists of a combination of
progress, collision and lane following terms. More details on learning,
hyperparameters, reward function, and baselines are provided in the
supplementary material.  

\begin{figure}[t]
  \centering
  \includegraphics[width=\textwidth]{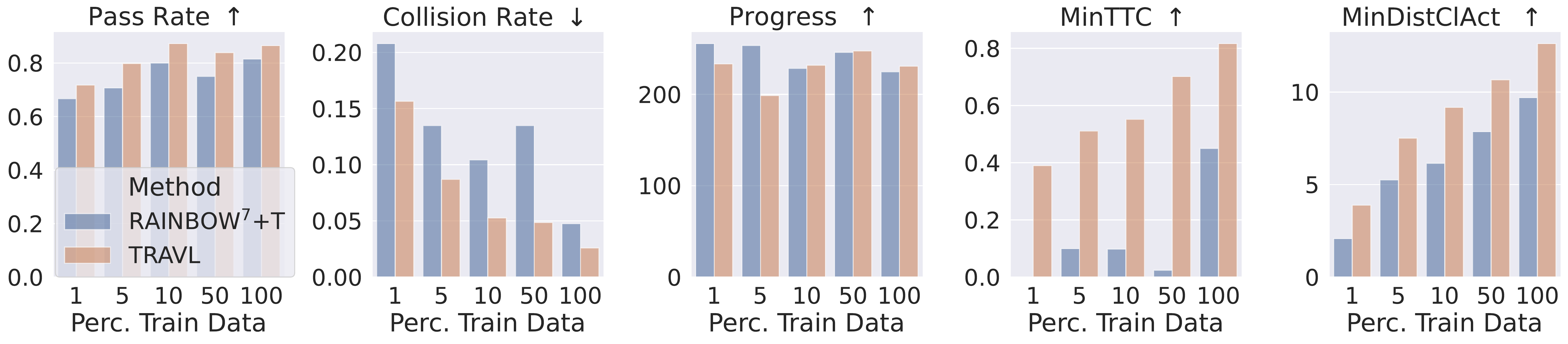}
  \caption{Increasing scenario diversity improves performance across the board.}
  \label{fig:metrics}
\end{figure}

\begin{wrapfigure}{r}{0.37\textwidth}
  \centering
  \includegraphics[width=0.35\textwidth]{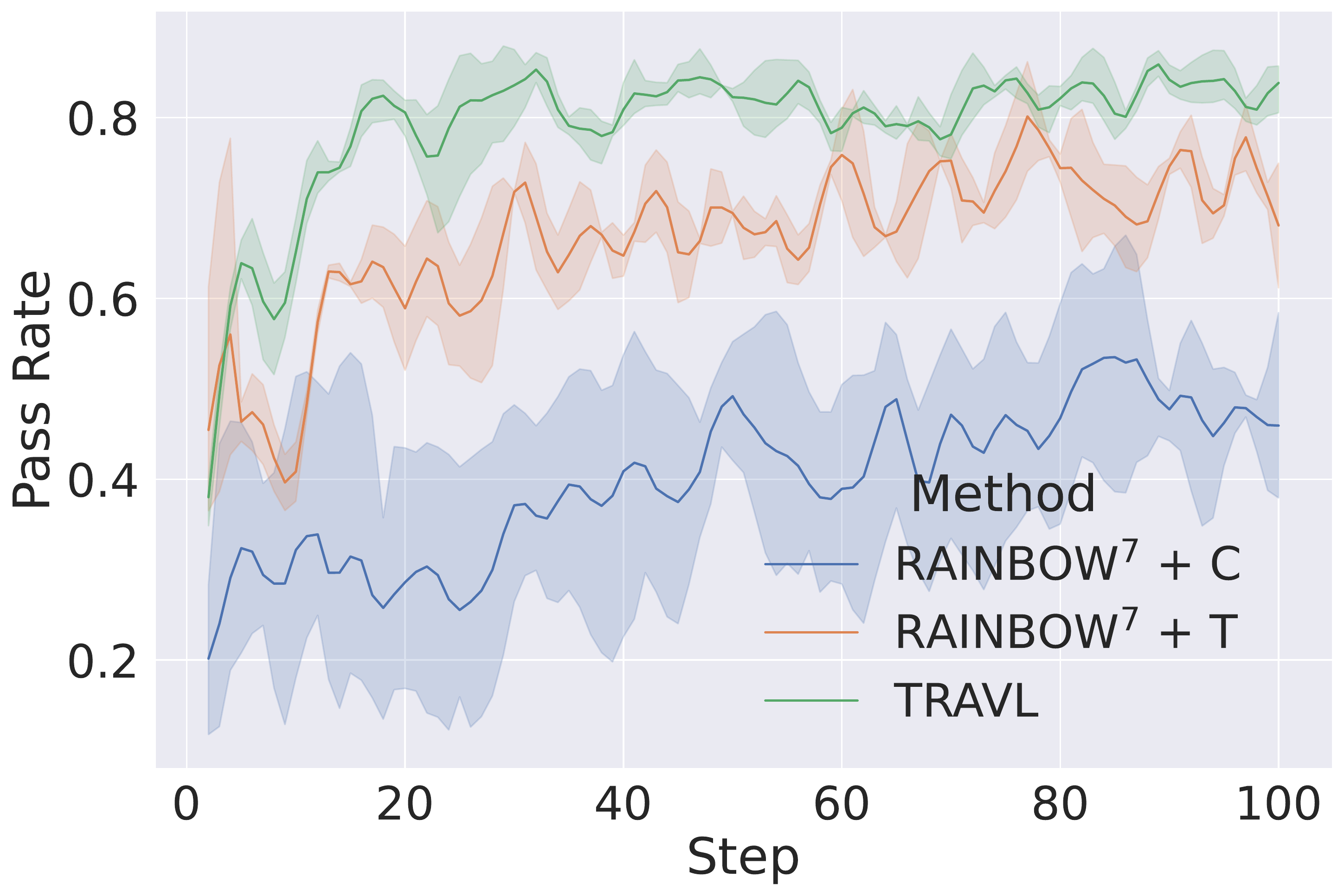}
  \caption{Training curves for 3 runs. TRAVL has the least variance and
    converges faster. }
  \label{fig:training-curve}
\end{wrapfigure}
As shown in Table.~\ref{table:sota}, trajectory-based methods outperform their
control-based counterparts, suggesting our proposed trajectory-based formulation
and architecture allow models to better learn long-term reasoning and benefit
from the trajectory sampler's inductive bias. We also see that
RAINBOW\textsuperscript{\ref{simp-rainbow}}+T outperforms other RL
trajectory-based baselines. This suggests that trajectory \emph{value learning}
is easier compared to learning a policy directly over the trajectory set.
Furthermore, its off-policy nature allows the use of a replay buffer for more
efficient learning. In contrast, on-policy methods such as A3C and PPO have
difficulty learning, \eg, it takes 10x longer to overfit to a single scenario
compared to off-policy approaches. This aligns with similar findings
in~\cite{dosovitskiy2017carla}. Additionally, IL baselines outperform weaker RL
baselines that suffer from issues of sample complexity, but falls short to more
efficient RL baselines due to the problem of distribution shift. Finally, TRAVL
outperforms the baselines. We believe this is because our  model-based
counterfactual loss provides denser supervisions and thus reduces noisy
variances during learning. This is empirically validated in
Fig.~\ref{fig:training-curve}  as our method has the least variance and
converges much faster.

\subsubsection{Targeted vs Free-flow:}
\label{sec:targeted-set}
We now study the effect of using different types of scenarios for training by
comparing models learned on our \emph{targeted} vs \emph{free-flow} set. As
shown in Table~\ref{table:target_set}, a model learned on the targeted set
performs the best. Notably this is \emph{even true when evaluating on the
  free-flow test set}, which is closer in distribution to the free-flow train set.
The effectiveness of the targeted set can be due to two reasons. Firstly, as
scenarios and actors are carefully designed, each targeted scenario is more
likely to provide interesting interactions, resulting in stronger learning
signals. On the contrary, many of the free-flow scenarios can be relatively
monotonous, with fewer interactions among actors. Secondly, the design of each
targeted scenario is driven by autonomous driving capabilities which where
determined with expert prior knowledge and are specifically meant to capture
what is necessary to be able to drive in nearly all scenarios. As a result, each
type of targeted scenario can be viewed as a basis over the scenario space.
Sampling behavioral variations results in a scenario set that  provides wide
coverage.\looseness=-1

\begin{table*}[t]
  \centering
  \setlength{\tabcolsep}{3pt}
  \begin{tabular}{c|ccccc}
    \specialrule{.2em}{.1em}{.1em}
    Method         & Pass Rate $\uparrow$ & Col. Rate $\downarrow$ & Prog.
    $\uparrow$     & MinTTC $\uparrow$    & MinDist $\uparrow$                    \\
    \hline
    Map Variation  & 0.738                & 0.070                  & 230
                   & 0.53                 & 8.97                                  \\
    Beh. Variation & \textbf{0.872}       & \textbf{0.022}         & 228
                   & 0.60                 & 9.82                                  \\
    Both           & 0.865                & 0.026                  & \textbf{231}
                   & \textbf{0.82}        & \textbf{12.6}                         \\
    \specialrule{.1em}{.05em}{.05em}
  \end{tabular}
  \caption{We train a \model agent on datasets with different axes of variation.
    Behavioral variation has larger effects than map for learning robust
    driving policies.\looseness=-1}
  \label{table:map-variation}
\end{table*}

\subsubsection{Behavioral scale and diversity:}
We now study how many scenarios we need for learning robust policies. Standard
RL setups use only a single environment (\eg, a fixed set of behavioral
parameters for non-ego agents) and rely on the stochastic nature of the policy
to collect diverse data. However, we find this is not enough. We train models on
datasets with varying amount  of scenario variations while keeping the total
number of training simulation steps constant. As shown in
Fig.~\ref{fig:metrics}, models trained with more diverse scenario variations
exhibit better performance. In particular, we see that while metrics
like pass rate and progress saturate quickly, safety-related metrics improve as
we increase the number of variations. This suggests that adding in data
diversity allows the model to be better prepared for safety-critical situations.
We also study which axis of variation has the largest effect. Specifically, we
disentangle map (\ie, geolocation) variations (road curvature, number of lanes,
topologies) and behavioral variation (actor triggers, target speeds) and
construct datasets that only contain one source of diversity while keeping the
total number of scenarios the same. Table~\ref{table:map-variation} shows that
\model is able to perform well even without map variations as our
trajectory-based formulation allows us to embed strong map priors into the
model. However, the behavioral variation component is crucial in learning more
robust policies.\looseness=-1

\section{Conclusion}
We have studied how to design traffic scenarios and scale training environments
in order to create an effective closed-loop benchmark for autonomous driving. We
have proposed a new method to efficiently learn driving policies which can
perform long-term reasoning and planning. Our method reasons in trajectory space
and can efficiently learn in closed-loop by leveraging additional imagined
experiences. We provide theoretical analysis and empirically demonstrate the
advantages of our method over the baselines on our  new benchmark.

\clearpage
\bibliographystyle{splncs04}
\bibliography{egbib}
\clearpage
\appendix

\newpage
\null
\vskip .375in
\begin{center}
  {\Large \bf Supplementary Material \par}
  \vspace*{24pt}
\end{center}
In this supplementary material, we provide  details about the implementation of
our method and benchmark, as well as more experimental analysis. In the
following, we introduce the backbone architecture and trajectory sampler in
Section~\ref{sec:supp-tech}. We then provide implementation details of learning
in Section~\ref{sec:supp-learning-objective} and reward functions in
Section~\ref{sec:supp-reward}. The proof of our theoretical analysis is
presented in Section~\ref{sec:supp-theory}. We also explain our benchmark
construction in more details in Section~\ref{sec:supp-benchmark}. We show some
additional quantitative  and qualitative analysis of TRAVL in
Section~\ref{sec:supp-quan} and Section~\ref{sec:supp-qual} respectively.
Finally, a high level overview this work can be found in the video
\texttt{rethinking_clt.mp4}.

\section{Technical Details of TRAVL}
\label{sec:supp-tech}
\subsection{Backbone Architecture}
Given an input rasterization tensor, our backbone network first performs three
layers of $3\times3$ convolution with 64 channels. It then applies 4 consecutive
\textit{ResBlock} units. Each block consists of a $1\times1$, $3\times3$ and
$1\times1$ convolutional layer as well as a skip-connection between input and
output features. The input channels of these 4 units are $(64, 256, 1024, 4096)$
respectively and the convolution kernels for each layer has the same number of
channels as the inputs, except for the last $1\times1$ layer that upsamples
channels for the next stage. Besides, the $3\times3$ layer in each unit has a
stride of 2 to downsample the feature map. Finally, we use two additional
$3\times3$ convolutional layers to reduce the channel number to $C=512$ without
further downsampling. This produces a final backbone feature map $\mathbf{F}\in
  \frac{H}{8} \times \frac{W}{8} \times 512$.

\subsection{Trajectory Sampler}
As stated in the main paper, we use a trajectory sampler which samples
longitudinal and lateral trajectories with respect to reference lanes. In
Figure~\ref{fig:trajectory_samples} we show a visualization of a trajectory
sample set. As our trajectory sampler considers map priors through the Frenet
frame, it can produce smooth trajectories compatible with the lane shapes. This
introduces inductive biases to driving maneuvers and is expected to ease the
learning.

\subsection{Planned vs. executed trajectory mismatch during MPC}
Because our method plans in an MPC fashion, an entire trajectory is selected as
the action but only the initial segment is executed, resulting in a mismatch.
One way we have tried to address this mismatch is through executing an entire
trajectory during rollout (instead of replanning in an MPC fashion) to collect
experience into the replay buffer, yet we didn't notice significant gains over
our current approach. One potential reason for this is that, even though using
the entire trajectory is less theoretically complex, it also significantly
reduces the number of simulated (state, action) pairs since an action now takes
longer simulation time to execute. With limited computation resources, such an
approach might degrade the performance due to less data.

\section{Learning}
\label{sec:supp-learning}
\subsection{Learning Objective}
\label{sec:supp-learning-objective}
Recall that our learning process alternates between the \textit{policy
  evaluation} and \textit{policy improvement} step. The policy evaluation step we
use is described in Eq.~\ref{eq:objective} in the main paper, as well as below
\begin{align}
  \label{eq:supp-objective}
  Q^{k+1} & \leftarrow \argmin_{Q_{\theta}} \mathbb{E}_{\mathcal{D}}\left[
    \underbrace{\left(Q_{\theta}(s, \tau) - \mathcal{B}_{\pi}^k Q^k(s, \tau)\right)^2}_{\text{Q-learning}} +
    \alpha_k \underbrace{\mathbb{E}_{\tau' \sim \mu(\tau'|s)}\left(R_{\theta}(s,
      \tau') - r'\right)^2 }_{\text{Counterfactual Reward Loss}}
  \right],\nonumber                                                                    \\
          & s.t. \quad Q_{\theta} = R_{\theta} + V_{\theta}, \quad V_{\theta} = \gamma
  \gP_{\pi}^kQ^k.
\end{align}
The policy improvement step is described in Eq.~\ref{eq:actor-critic} in the main paper, as
well as below
\begin{align}
  \label{eq:supp-policy-improvement}
   & \pi^{k+1} \leftarrow (1-\epsilon)\argmax_{\pi}\E_{s\sim\mathcal{D},
    a\sim\pi}[Q^{k+1}(s,a)] + \epsilon U(a),
\end{align}
For the \textit{policy evaluation} step, we use SGD to optimize an empirical
objective (\ie, a mini-batch estimation) of Eq.~\ref{eq:supp-objective}. Note
that Eq.~\ref{eq:supp-objective} can be converted as a Lagrangian term.
Essentially, we take gradient steps of the following loss function $\mathcal{L}$
over parameter $\theta$ to obtain the optimal solution for
Eq.~\ref{eq:supp-objective},
\begin{align}
  \label{eq:loss}
  \mathcal{L} = \frac{1}{|\mathcal{D}|}\sum_{(s, \tau, r, s')} & \left[
  \underbrace{\left(R_{\theta}(s, \tau) + V_{\theta}(s, \tau) - r - \gamma Q^k(s',
  \pi^k(\tau'|s'))\right)^2}_{Q-learning} \right.\nonumber                                                                                             \\
                                                               & + \underbrace{\alpha \frac{1}{|\mathcal{T}|}\sum_{\tau'\neq\tau, \tau'\in\mathcal{T}}
  (R_{\theta}(s, \tau') - r')^2}_{Counterfactual-loss}\nonumber                                                                                        \\
                                                               & \left.+ \underbrace{\lambda \left(V_{\theta}(s, \tau) - \gamma Q^k(s',
    \pi^k(\tau'|s'))\right)^2}_{Lagrangian-loss}
  \right].
\end{align}
However, this involves an expensive double loop optimization, \ie, outer loop
for iterating $Q^k$ and inner loop for minimizing $\mathcal{L}$. We hence simply
apply one gradient step for the inner loop updating $Q^k$ to $Q^{k+1}$. In
practice, we also found using a cross-entropy loss to minimize the KL-divergence
between $e^{-R_{\theta}}$ and $e^{-r'}$ helps stabilize training compared to
using $\ell_2$ loss for the counterfactual term, possibly because the former one
is less prone to outlier estimation of $r'$ which is caused by our approximated
world modeling. Our overall learning process is summarized in
Algorithm~\ref{algo:learning}.

\subsubsection{Implementation Details:} We train our method using the Adam
optimizer~\cite{kingma2014adam}. We use a batch size of 10 and learning rate of
0.0001. To accelerate training, we collect simulation data asynchronously with 9
instances of the simulator and store them in a prioritized replay buffer. We
initialize the $\epsilon$ as 0.1 and linearly decay it to $0.01$ in the first
200k steps, and terminate learning at 1 million steps as the model converges.
Besides, we use $\gamma=0.95$, $\alpha=1.0$ and $\lambda=0.01$.

Our imitation learning baselines use the same input representation and network
as our model. We replace the learning loss with $\ell_2$ loss for the control
based model and max-margin for the trajectory sampling based
model~\cite{zeng2019end}. We use a well-tuned auto-pilot model as our expert
demonstration, which has access to all ground-truth world states in the
simulation. Note that we use rasterization of detection boxes as inputs for all
approaches. Therefore the baselines are similar to the privileged agent in
LBC~\cite{chen2020learning}.

\begin{figure}[H]
  \centering
  \includegraphics[width=\textwidth]{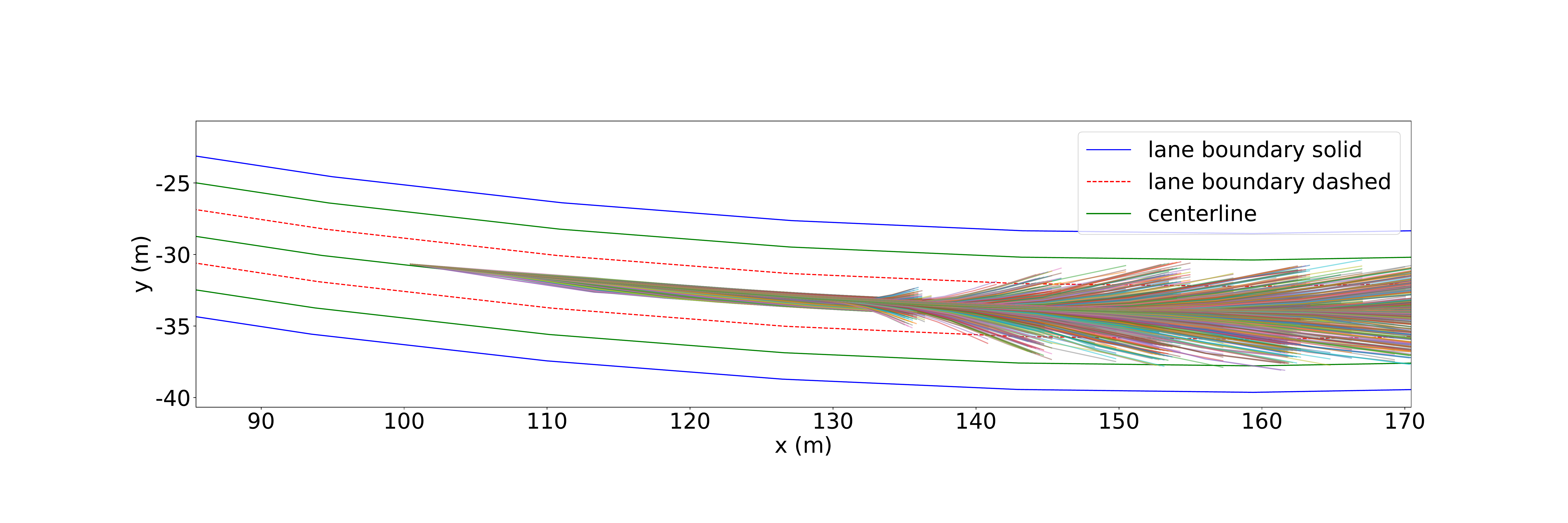}
  \caption{Example samples from our trajectory sampler which uses map
    information. }
  \label{fig:trajectory_samples}
\end{figure}

\subsection{Reward Function}
\label{sec:supp-reward}
Our reward function $R$ is a linear combination of \textit{progress},
\textit{collision}, and \textit{lane following} terms:
\[R(s^t, \tau^t, s^{t+1}) =C_p\cdot R_p(s^t, \tau^t, s^{t+1})+C_c\cdot R_c(s^t,
  \tau^t, s^{t+1})+C_l\cdot R_l(s^t, \tau^t, s^{t+1}),  \] where $C_p=0.6,
  C_c=40.0, C_l=1.0$ are constants designed to balance the relative scale between
the reward terms. $R_p$ is the progress term, and rewards the agent for distance
traveled along the goal lane. Here, a goal lane is defined by an external router
and we assume to have access to it. We use $D_{travel}$ to denote the traveled
distance between the projections of $s^t$ and $s^{t+1}$ on the goal lane, and
$D_{lane}$ to denote the distance between $s^t$ and its projection. The
\textit{progress reward} is defined as $R_p =  e^{-0.2 \times
      D_{lane}}D_{travel}$, where the term $e^{-0.2 \times D_{lane}}$ penalizes the
agent for driving further from the goal lane ($D_{lane}$). $R_c$ is a term
penalizing the agent for collisions, and is defined as:
\begin{equation*}
  R_c(s^t, \tau^t, s^{t+1}) = \left\{
  \begin{array}{ll}
    -1.0 & \text{if the agent has collided at $s^{t+1}$}, \\
    0.0  & \text{otherwise}.
  \end{array}
  \right.
\end{equation*}
Finally, $R_l$ is a lane following term penalizing the agent for deviating from
the goal lane. For an action $\tau = \{(x^0, y^0), (x^1, y^1), \cdots (x^T,
  y^T)\}$, $R_l$ is defined as the sum of the negative distances between each
$(x^i, y^i)$ and its projection on the goal lane.

\begin{algorithm}[t]
  \caption{TRAVL: TRAjectory Value Learning}
  \label{algo:learning}
  \begin{algorithmic}[1]
    \REQUIRE Simulator, Training Scenario Set \INITIALIZE $\mathcal{D} \leftarrow
      \emptyset$, $\pi(\tau|s) \leftarrow \text{Uniform$(\tau)$}$, TRAVL network
    $\leftarrow$ random weights.
    \rollout
    \WHILE {Learning has not ended} \STATE Sample a scenario variation from the
    training scenario set. \STATE Produce $(s^t, \tau^t, r^t, s^{t+1})$ by
    interacting the policy $\pi$ and the simulator on the sampled scenario.
    \STATE Store $(s^t, \tau^t, r^t, s^{t+1})$ to the replay buffer
    $\mathcal{D}$. \ENDWHILE

    \learning
    \FOR{$k=0, \cdots, $ max\_iter} \STATE Draw (mini-batch) samples $(s^t,
      \tau^t, r^t, s^{t+1})$ from $\mathcal{D}$. \STATE Draw a set of trajectory
    samples $\mathcal{T}$ given $s^t$. \STATE Compute $R_{\theta}(s^t, \tau^t)$,
    $V_{\theta}(s^t, \tau^t)$ and $R_{\theta}(s^t, \tau')$, $V_{\theta}(s^t,
      \tau')$ for $\tau' \in \mathcal{T}$ using TRAVL network. \STATE Evaluate $r'
      = R(s^t, \tau', s')$ for $\tau'\in\mathcal{T}$ using reward functions.
    \STATE Compute $\mathcal{L}$ using Eq.~\ref{eq:loss}. \STATE Update network
    parameter $\theta$ using gradients of $\mathcal{L}$. \STATE $Q_{\theta}
      \leftarrow R_{\theta} + V_{\theta}$. \STATE $\pi(\tau|s) \leftarrow \left\{
      \begin{array}{ll}
        \argmax_{\tau} Q_{\theta}(s, \tau), & \text{with probability
        $1-\epsilon$}                                                             \\
        \text{randomly sample $\tau$},      & \text{with probability $\epsilon$}.
      \end{array}
      \right. $

    \ENDFOR

  \end{algorithmic}
\end{algorithm}

\section{Theoretical Analysis}
\label{sec:supp-theory}
\begin{lemma}
  Assuming $R$ is bounded by a constant $R_{\text{max}}$ and $\alpha_k$
  satisfies
  \begin{equation}
    \alpha_k <
    \left(\frac{1}{\gamma^kC}
    -1\right)^{-1}\left(\frac{\pi^k}{\mu}\right)_{min},
  \end{equation}
  with $C$ an arbitrary constant,  iteratively applying
  Eq.~\ref{eq:supp-objective} and the policy update step in
  Eq.~\ref{eq:supp-policy-improvement} converges to a fixed point.
\end{lemma}
\begin{proof}
  To prove lemma 1 is correct, it suffices to show that the updating rule in
  Eq.~\ref{eq:supp-objective} leads to $\lim_{k=\infty}||Q^{k+1} -
    Q^k||_{\infty} = 0$. To find out the optimal $Q_{\theta}$ at iteration $k$, we
  take the derivative of the R.H.S. and set it to $0$ as follows
  \small
  \begin{align*}
     & \mathbb{E}_{\tau\sim\pi^k}\left[Q_{\theta}(s, \tau) -
    \mathcal{B}^k_{\pi}Q^k(s, \tau)\right] + \alpha_k
    \mathbb{E}_{\tau'\sim\mu}\left[
      Q_{\theta}(s, \tau') - \gamma\gP_{\pi}^k
      Q^k(s,\tau') - \mathbb{E}_{\tau\sim\pi^k}r'
      \right] = 0.
    \nonumber                                                \\
  \end{align*}
  \normalsize
  Now we will interchange $\tau$ and $\tau'$ in the second term of the equation
  above ($\alpha_k\mathbb{E}_{\tau'\sim\mu}[\cdots]$) and use the fact that
  $\mathbb{E}_{\mu}[\cdots] = \mathbb{E}_{\pi}[\frac{\mu}{\pi}\cdots]$ to
  obtain.
  \small
  \begin{align*}
     & \mathbb{E}_{\tau\sim\pi^k}\left[Q_{\theta}(s, \tau) -
    \mathcal{B}^k_{\pi}Q^k(s, \tau) +     \frac{\alpha_k
      \mu}{\pi^k}Q_{\theta}(s, \tau) - \frac{\alpha_k
      \mu}{\pi^k}\gamma\gP_{\pi}^k
    Q^k(s,\tau) - \frac{\alpha_k\mu}{\pi^k}\mathbb{E}_{\tau'\sim\pi^{k}}r'
    \right] = 0.
    \nonumber                                                \\
  \end{align*}
  \normalsize
  Thus by definition of $Q^{k+1}$ in Eq.~\ref{eq:supp-objective}, we have
  \begin{align}
     & \Rightarrow    Q^{k+1}(s, \tau) = Q_{\theta}(s, \tau) \nonumber \\
     & = \frac{\pi^k}{\pi^k + \alpha_k
      \mu}\mathcal{B}^{k}_{\pi}Q^k(s, \tau) + \frac{\alpha_k\mu}{\pi^k + \alpha_k
      \mu}\gamma\gP_{\pi}^kQ^k(s, \tau) + \frac{\alpha_k\mu}{\pi^k + \alpha_k
      \mu}\mathbb{E}_{\tau'\sim\pi^{k}}r'.
  \end{align}
  Note that $\mathcal{B}^k_{\pi} = R + \gamma\gP^k_{\pi}$, and we can  further
  simplify $Q^{k+1}$ as
  \begin{equation}
    \label{eq:update_formula}
    Q^{k+1} = \mathcal{B}^k_{\pi}Q^k(s, \tau) - \frac{\alpha_k\mu}{\pi^k +
      \alpha_k\mu}\left[R(s, \tau) - \mathbb{E}_{\tau'\sim\pi^k}r'\right].
  \end{equation}
\end{proof}
Now, we only need to show Eq.~\ref{eq:update_formula} leads to
$||Q^{k+1}-Q^k||_{\infty}\rightarrow 0$. First, it can be shown that
$||\mathcal{B}^k_{\pi}Q^{k} - \mathcal{B}^{k-1}_{\pi}Q^{k-1}||_{\infty} \leq
  \gamma||Q^k - Q^{k-1}||_{\infty}$ following \cite{melo2001convergence}.
Therefore we have

\small
\begin{align}
  ||Q^{k+1} - Q^k||_{\infty}
   & \leq ||\mathcal{B}^k_{\pi}Q^{k} -
  \mathcal{B}^{k-1}_{\pi}Q^{k-1}||_{\infty} \nonumber                                         \\
   & \quad + ||\frac{\alpha_k\mu}{\pi^k + \alpha_k\mu}[R - \mathbb{E}_{\pi^k}r']
  - \frac{\alpha_{k-1}\mu}{\pi^{k-1} + \alpha_{k-1}\mu}[R -
  \mathbb{E}_{\pi^{k-1}}r']||_{\infty}\nonumber                                               \\
   & \leq\gamma||Q^k-Q^{k-1}||_{\infty} + 2R_{max} \left(||\frac{\alpha_{k-1}\mu}{\pi^{k-1} +
    \alpha_{k-1}\mu}||_{\infty} + ||\frac{\alpha_{k}\mu}{\pi^{k} +
  \alpha_{k}\mu}||_{\infty}\right).\nonumber                                                  \\
\end{align}
\normalsize
Using $ \alpha_k < \left(\frac{1}{\gamma^kC}
  -1\right)^{-1}\left(\frac{\pi^k}{\mu}\right)_{min} $, we have
\begin{align}
  ||Q^{k+1} - Q^k||_{\infty}
   & < \gamma||Q^k-Q^{k-1}||_{\infty} + 2R_{max} \left(\gamma^{k-1}C + \gamma^kC\right) \\
   & < \gamma \left[\gamma ||Q^{k-1}-Q^{k-2}||_{\infty} +
  2R_{max}\left(\gamma^{k-2}C + \gamma^{k-1}C\right)\right] \nonumber                   \\
   & \quad + 2R_{max} \left(\gamma^{k-1}C + \gamma^kC\right)                            \\
   & \cdots\nonumber                                                                    \\
   & <\gamma^{k}||Q^1-Q^0||_{\infty} + 2R_{max}C(1+\gamma)k\gamma^{k-1}.
\end{align}
Therefore, we have $$\lim_{k \rightarrow \infty}||Q^{k+1}-Q^k||_{\infty} = 0.$$
$\square$.

\begin{theorem}
  Under the same conditions as Lemma 1, our learning procedure converges to
  $Q^*$.
\end{theorem}
\begin{proof}
  To show the sequence of $Q^k$ converges to $Q^*$, we first show that $Q^{k+1}$
  is sufficiently close to the following value $\hat{Q}^{k+1}$ when $k$ is
  large,
  \begin{align}
    Q^{k+1} \rightarrow \hat{Q}^{k+1} \coloneqq \left( I - \gamma
    \gP_{\pi}^k\right)^{-1} \left[
      \frac{\pi^k}{\pi^k + \alpha_k\mu}R + \frac{\alpha_k\mu}{\pi^k +
        \alpha_k\mu}\mathbb{E}_{\pi^k}r'
      \right].
  \end{align}
  To see this, we take a subtraction between $(I - \gamma\gP_{\pi}^k)Q^{k+1}$
  and $(I-\gamma\gP_{\pi}^k)\hat{Q}^{k+1}$. We have,
  \begin{align}
    ||\left( I - \gamma \gP_{\pi}^k\right) \left(Q^{k+1} - \hat{Q}^{k+1}\right) ||_{\infty}
     & = ||\gamma \gP_{\pi}^k Q^k - \gamma \gP_{\pi}^k Q^{k+1} ||_{\infty}\nonumber \\
     & = \gamma ||\gP_{\pi}^k \left(Q^k - Q^{k+1}\right) ||_{\infty}. \nonumber     \\
  \end{align}
  Note that $\gP_{\pi}^k$ is the transition matrix coupled with policy $\pi$.
  This means that for arbitrary matrix $A$, $||\gP_{\pi} A||_{\infty} \leq
    ||A||_{\infty}$. Therefore, we have
  \begin{equation}
    ||\left( I - \gamma \gP_{\pi}^k\right) \left(Q^{k+1} - \hat{Q}^{k+1}\right) ||_{\infty}
    \leq \gamma ||\left(Q^k - Q^{k+1}\right) ||_{\infty}\rightarrow 0.
  \end{equation}
  Besides, it is also easy to see that
  \small
  \begin{align}
    ||\left( I - \gamma \gP_{\pi}^k\right) \left(Q^{k+1} - \hat{Q}^{k+1}\right) ||_{\infty}
     & =
    ||\left(Q^{k+1} - \hat{Q}^{k+1}\right) - \gamma \gP_{\pi}^k\left(Q^{k+1} - \hat{Q}^{k+1}\right) ||_{\infty}
    \nonumber                                                                 \\
     & \geq
    ||\left(Q^{k+1} - \hat{Q}^{k+1}\right) ||_{\infty} -
    ||\gamma \gP_{\pi}^k\left(Q^{k+1} - \hat{Q}^{k+1}\right) ||_{\infty}
    \nonumber                                                                 \\
     & \geq \left(1 - ||\gamma\gP_{\pi}^k||_{\infty}\right) ||\left(Q^{k+1} -
    \hat{Q}^{k+1}\right) ||_{\infty}.
  \end{align}
  \normalsize
  Again, since $\gP_{\pi}^k$ is the transition probability matrix, we know $(1 -
    ||\gamma\gP_{\pi}^k||_{\infty}) > 0$. Hence, we have
  \begin{align}
     & \left(1 - ||\gamma\gP_{\pi}^k||_{\infty}\right) ||\left(Q^{k+1} -
    \hat{Q}^{k+1}\right) ||_{\infty}
    \leq \gamma ||\left(Q^k - Q^{k+1}\right) ||_{\infty}\rightarrow 0.
    \nonumber                                                            \\
     & \Rightarrow
    Q^{k+1} \rightarrow \hat{Q}^{k+1}.
  \end{align}
  When $k \rightarrow \infty$, given this fact and $\alpha_k \rightarrow 0$, we
  have
  $$ Q^{\infty} = \left(I - \gamma \gP^{\infty}\right)R. $$ Note that this is
  exactly the fixed point of the standard Bellman operator, \ie, $Q^* =
    \mathcal{B}Q^* = R + \gamma \gP^*Q^*$. Therefore, we know $Q^{\infty} = Q^*$.
  $\square$.

\end{proof}

\section{Benchmark Dataset}
\label{sec:supp-benchmark}
This section provides additional details about the free-flow and targeted
scenarios we use for our benchmark datasets, including how we generate and split
scenarios into train, validation and test sets.

\subsubsection{Free-flow Scenarios:}
Our free-flow dataset aims to model nominal traffic conditions, and consists of
7 scenario types. Differences in these scenario types include having more or
less aggressive actors, actors making fewer lane changes, a larger proportion of
large vehicles (e.g., trucks), faster actors, and larger variations in actor
speed. Each scenario type is defined by specifying a distribution over the
ego-vehicle's initial state (e.g., speed, location), actor speeds, actor classes
(e.g., car, bus, truck), actor IDM~\cite{treiber2000congested} profile (e.g.,
aggressive, cautious), and actor MOBIL~\cite{kesting2007general} profile (e.g.,
selfish, altruistic). Additional parameters configure actor density and the map
(e.g., map layout, number of lanes). Sampling a free-flow scenario amounts to
first uniformly sampling a scenario type and then sampling the scenario-defining
parameters from the aforementioned distributions.

\subsubsection{Targeted Scenarios:}
Our targeted scenario set consists of 24 distinct scenario types covering 3
common ego-routing intentions for highway driving. Scenarios corresponding to
different ego intentions have different success criteria:
\begin{enumerate}
  \item Lane Follow: Ego-vehicle must reach a goal point in the lane without
        deviating from the lane.
  \item Lane Change: Ego-vehicle must make a lane change towards a target lane
        and then reach a goal point on the target lane.
  \item Lane Merge: Ego-vehicle is driving on a lane that is ending and must
        merge into another lane.
\end{enumerate}
Besides, any collision or speed limit violation happens during the scenario also
accounts as a failure. To generate diverse traffic scenarios, the 3
aforementioned scenes can be combined with zero or more actors, where each actor
can be scripted with 1 of 5 behavior patterns (braking, accelerating, blocking
lane, cut into lane, negotiating lane change). A concrete example of a scenario
type is a lane follow scenario where an actor is cutting in front of the
ego-vehicle from another lane. Through varying the ego-routing intention, actor
behaviors, and actor placements, we designed 24 scenario types for our targeted
scenario set, which aim to cover the space of traffic interactions that would be
encountered during driving.

Each scenario type is parameterized by a set of behavioral and map parameters,
and an endless amount of scenario variations can be generated through varying
these parameters. Behavioral parameters control the details of the interaction
between the ego-vehicle and other actors, such as initial relative speeds,
initial relative distances, and how an actor performs a maneuver (e.g.,
aggressiveness of cut-in). Map parameters control the layout of the map such as
the curvature of the roads, the geometry of a merging lane, and the number of
lanes.

Note that while the process manually designing scenarios require human effort
(\eg compared to learned or adversarial-based approaches,) we'd like to
highlight that such a creation process encodes prior knowledge and makes the
created scenarios more semantically meaningful, as each type of scenario targets
a specific capability or requirement of autonomous driving. This ensures we have
a good coverage of real-world traffic situations. Our scenarios can also adapt
to different AV policies since we use intelligent actors and smart triggers
which can adjust automatically depending on the AV's maneuvers.

\subsubsection{Creating Dataset Splits:}
As described in the benchmark section of the main text (Section~\ref{sec:benchmark}), we use
the all-pairs methodology to construct our test set for targeted scenarios.
While enumerating all possible parameter combinations thoroughly covers the
scenario space, it is too expensive as the number of combinations grows
exponentially with the number of configurable parameters. All-pairs produces a
much smaller set by carefully selecting the combinations of parameter
variations~\cite{microsoft-pict}, \ie a set that ensures all possible
combinations of variations for any pair of parameters are presented. The
assumption behind this approach is that many interesting behaviors can be
triggered by changing a single parameter or a pair of parameters. As a result, a
test set with this property provides good coverage of the input space.

However, the standard all-pairs approach assumes that all parameters are
discrete, whereas many of our scenario parameters are continuous. To this end,
we partition each of our continuous scenario parameters into non-overlapping
buckets (a contiguous range of values). For example, the time an actor takes to
cut in front of the ego-vehicle is a continuous parameter. We can bucket the
values for this parameter into $[1, 2]$ seconds, $[3, 4]$ seconds and $[5,6]$
seconds, changing the semantics of the cut-in behavior from aggressive to mild.
This essentially discretizes continuous variables into coarse-grained discrete
variables, upon which the all-pairs approach can be applied. Once the discrete
choice of which bucket to use has been made for a scenario's continuous
parameters, we generate the exact value of each such parameter by uniform
sampling within the selected bucket.

\section{Metrics Breakdown}
\label{sec:supp-quan}
In this section we show the metrics in Table~\ref{table:sota} of the main paper broken
down by scenario types to provide more fine-grained analysis. Specifically, we
categorize scenarios in our targeted set into normal, negotiating and reacting
scenarios. Normal scenarios are those nominal scenarios such as lane following
with normal-behaving actors (driving in their lane without any extra maneuvers)
in the scene. Negotiating scenarios require negotiations with other actors, such
as squeeze-in lane changes and merges.  Finally, reacting scenarios are those
where the ego-vehicle must react to another actor, \eg an actor cutting in.

From Figure~\ref{fig:metrics-scenario-family}, we see that most methods are able
to achieve low collision rate and satisfactory progress on the normal scenarios.
However, for more complex negotiating and reacting scenarios, baseline methods
have difficulty exhibiting safe and efficient driving maneuvers. Specifically,
we can see that control signal based methods have very high collision rate on
difficult scenarios, possibly due to the lack of long-term reasoning. Second,
on-policy RL techniques such as PPO and A3C cannot achieve good performance.
Note that although the policy learned by A3C+T has low collision, it is too
conservative and does not progress very much compared to other methods. Finally,
combining our trajectory based formulation and off-policy learning achieves
better performance, \eg, RAINBOW+T, and our TRAVL is even better with the
proposed efficient learning paradigm.

\begin{figure}[t]
  \centering
  \includegraphics[width=\textwidth]{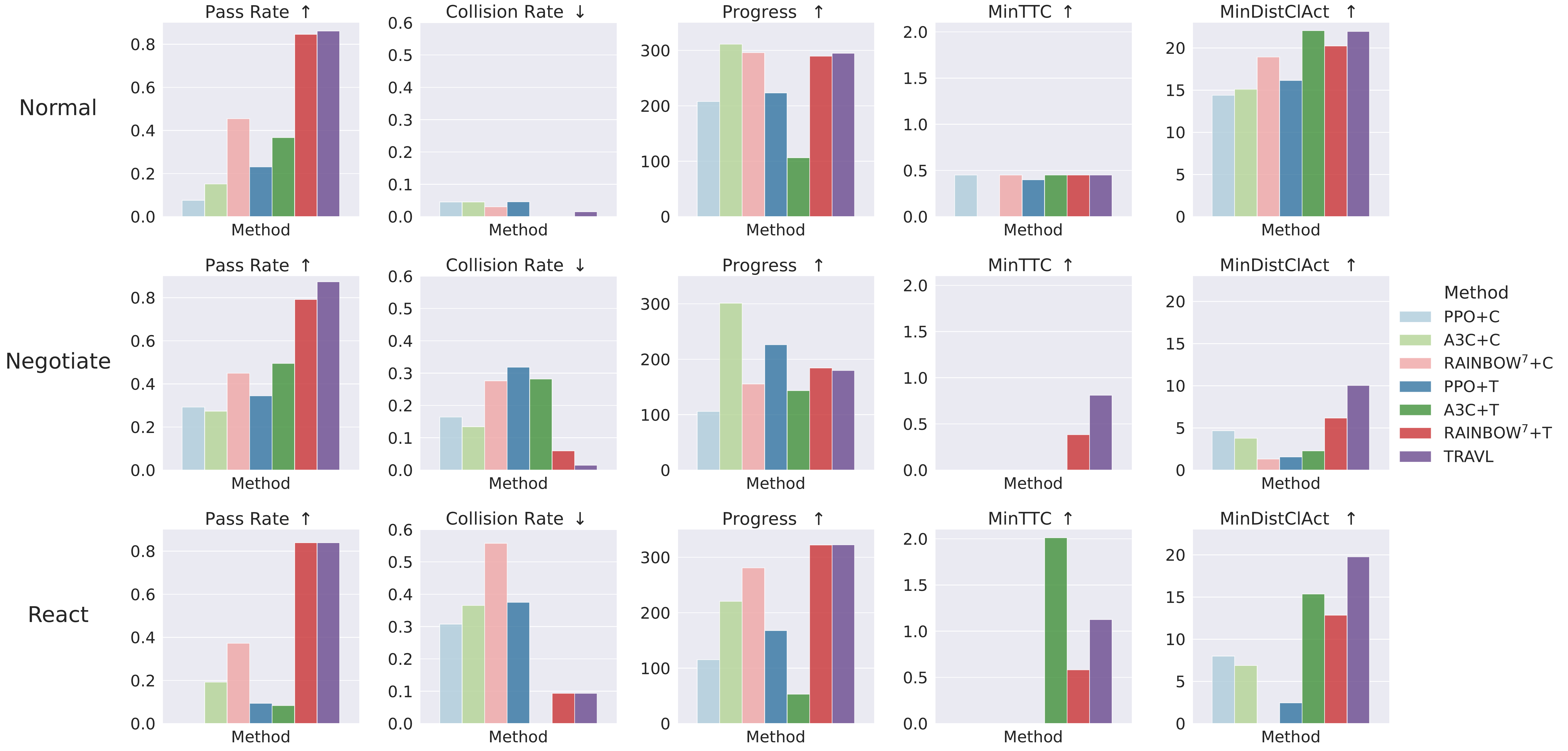}
  \caption{Metrics broken down by scenario types. Top row shows metrics for
    Normal scenarios. Middle row shows metrics for Negotiating scenarios. Bottom
    row shows metrics for Reacting scenarios. We see that while control based
    methods can avoid collision for Normal scenarios, Reacting scenarios prove more
    challenging. }
  \label{fig:metrics-scenario-family}
\end{figure}

\section{Qualitative Results}
\label{sec:supp-qual}
In this section, we show several qualitative results of our learned TRAVL agent
navigating in closed loop simulation. The agent is controlling the center pink
vehicle. Beside each actor is the velocity in $m/s$. Below velocity,
acceleration in $m/s^2$ is shown.

In Figure~\ref{fig:free-flow-ex}, the agent is driving in scenario where it must
merge onto the highway while taking into account other actors in the scene. We
see that our agent has learned to drive in complex free-flow traffic situations
which mimic the real world.

In Figure~\ref{fig:cut-in-ex}, we see our agent in a targeted scenario which
tests the ability to react to actors cutting in. We see our agent performs the
correct maneuver by reacting quickly and slowing down.

In Figure~\ref{fig:lane-change-ex}, the agent is tasked to squeeze between the
two actors. We see the agent has learned to slow down in order to make this lane
change.

In Figure~\ref{fig:merge-ex}, this scenario stress tests the agent by
initializing it at a very low velocity and requiring it to merge into a crowded
lane. We see the agent has learned to speed up to in order to merge into the
traffic.

In Figure~\ref{fig:lane-change-fail-ex}, we see a failure case of our model. In
this lane change scenario, we see a fast-travelling actor decelerating. The
ego-vehicle mistakenly initiates a lane change in front of that actor when that
actor is still going much too fast. Once our agent realizes that the actor
cannot slow down in time and that this will cause a collision, it makes a last
minute adjustment to avoid collision. While collision is avoided, this is still
an unsafe behavior.

\begin{figure}[H]
  \begin{subfigure}{0.32\textwidth}
    \centering
    \includegraphics[width=\linewidth,trim={24cm 10cm 24cm 10cm},clip]{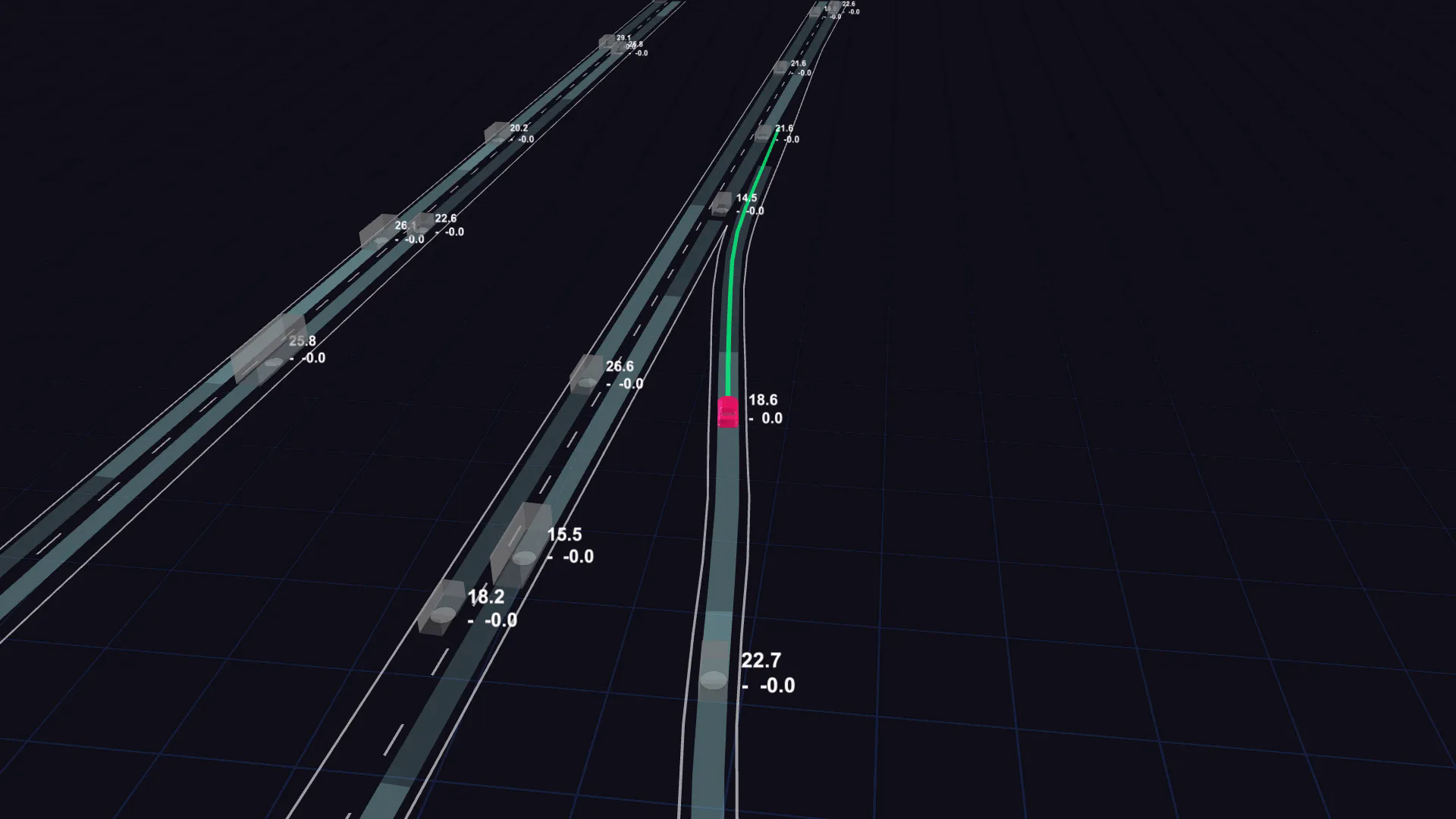}
  \end{subfigure}
  \begin{subfigure}{0.32\textwidth}\
    \centering
    \includegraphics[width=\linewidth,trim={24cm 10cm 24cm 10cm},clip]{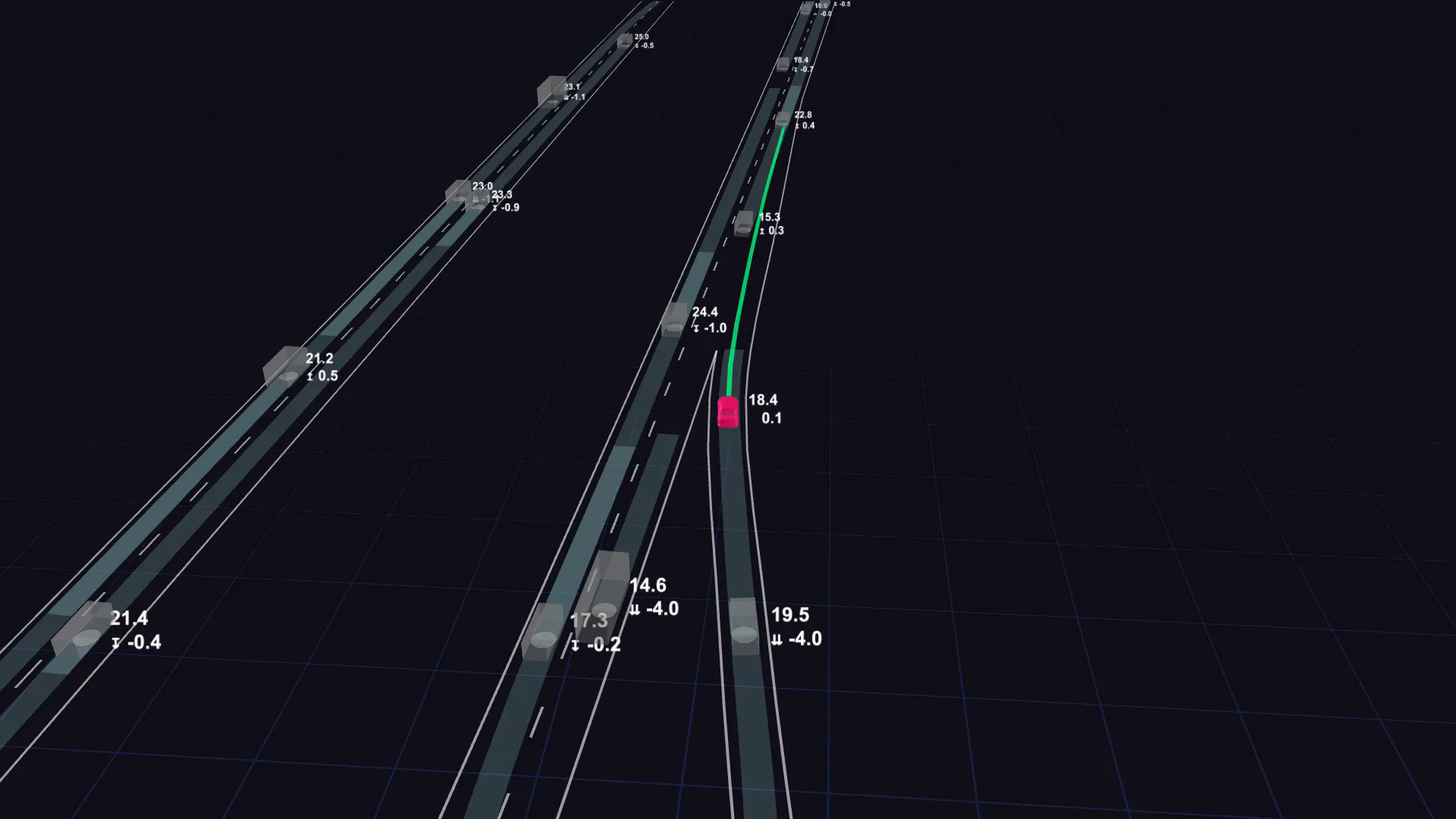}
  \end{subfigure}
  \begin{subfigure}{0.32\textwidth}
    \centering
    \includegraphics[width=\linewidth,trim={24cm 10cm 24cm 10cm},clip]{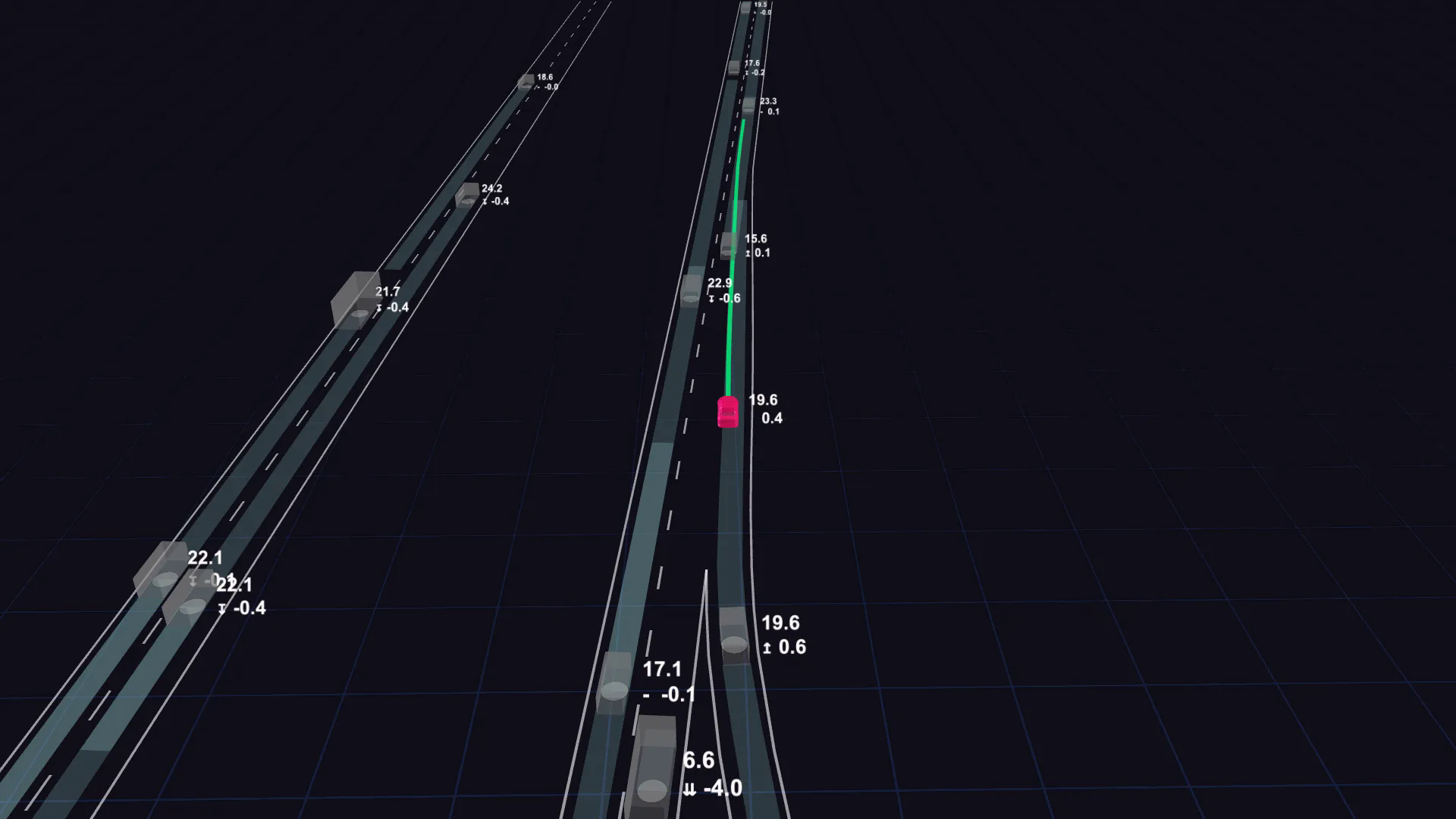}
  \end{subfigure}
  \caption{Our agent successfully navigates a free-flow scenario where it must merge.}
  \label{fig:free-flow-ex}
\end{figure}

\begin{figure}[H]
  \begin{subfigure}{0.32\textwidth}
    \centering
    \includegraphics[width=\linewidth,trim={24cm 10cm 24cm 10cm},clip]{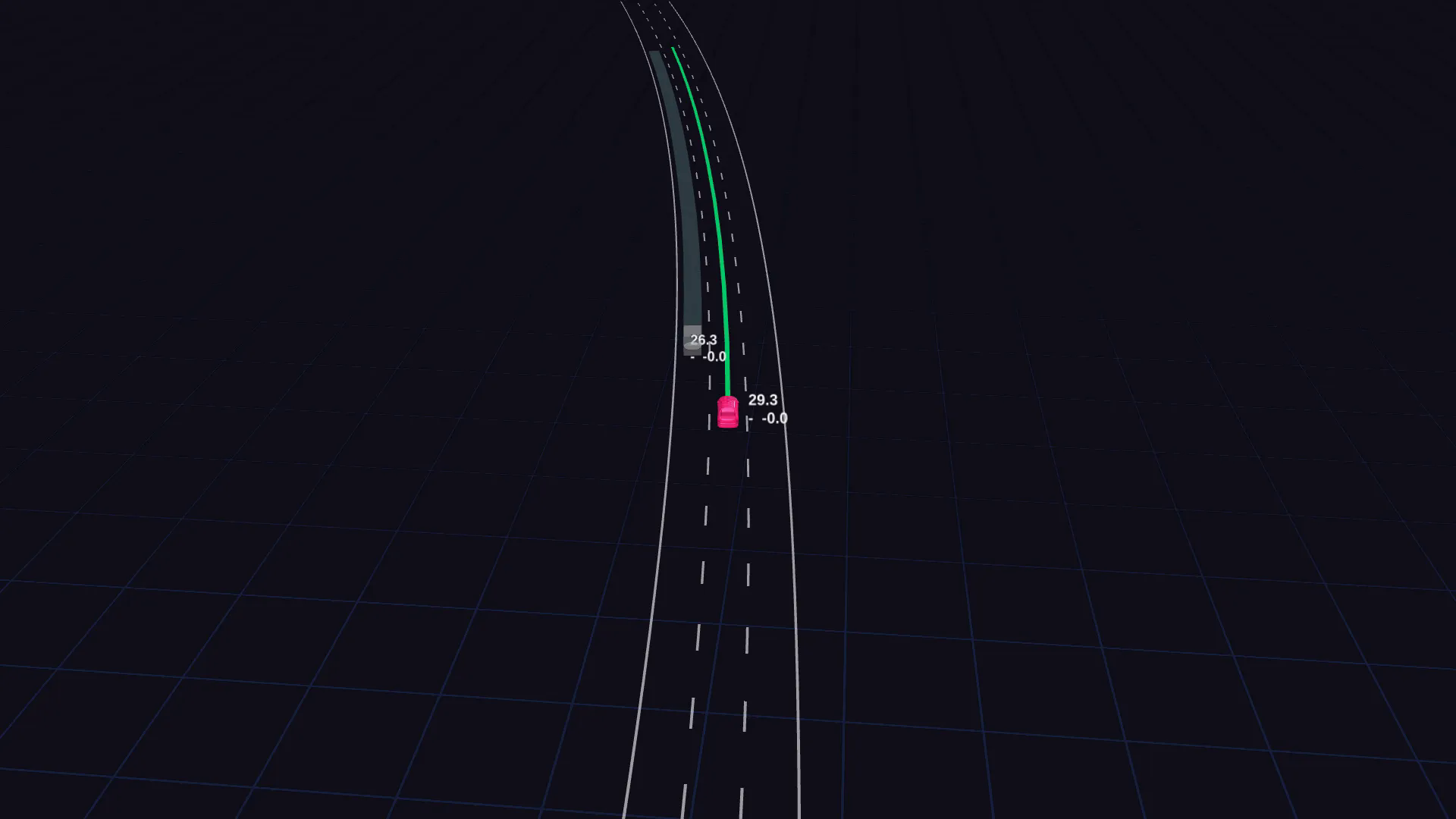}
  \end{subfigure}
  \begin{subfigure}{0.32\textwidth}\
    \centering
    \includegraphics[width=\linewidth,trim={24cm 10cm 24cm 10cm},clip]{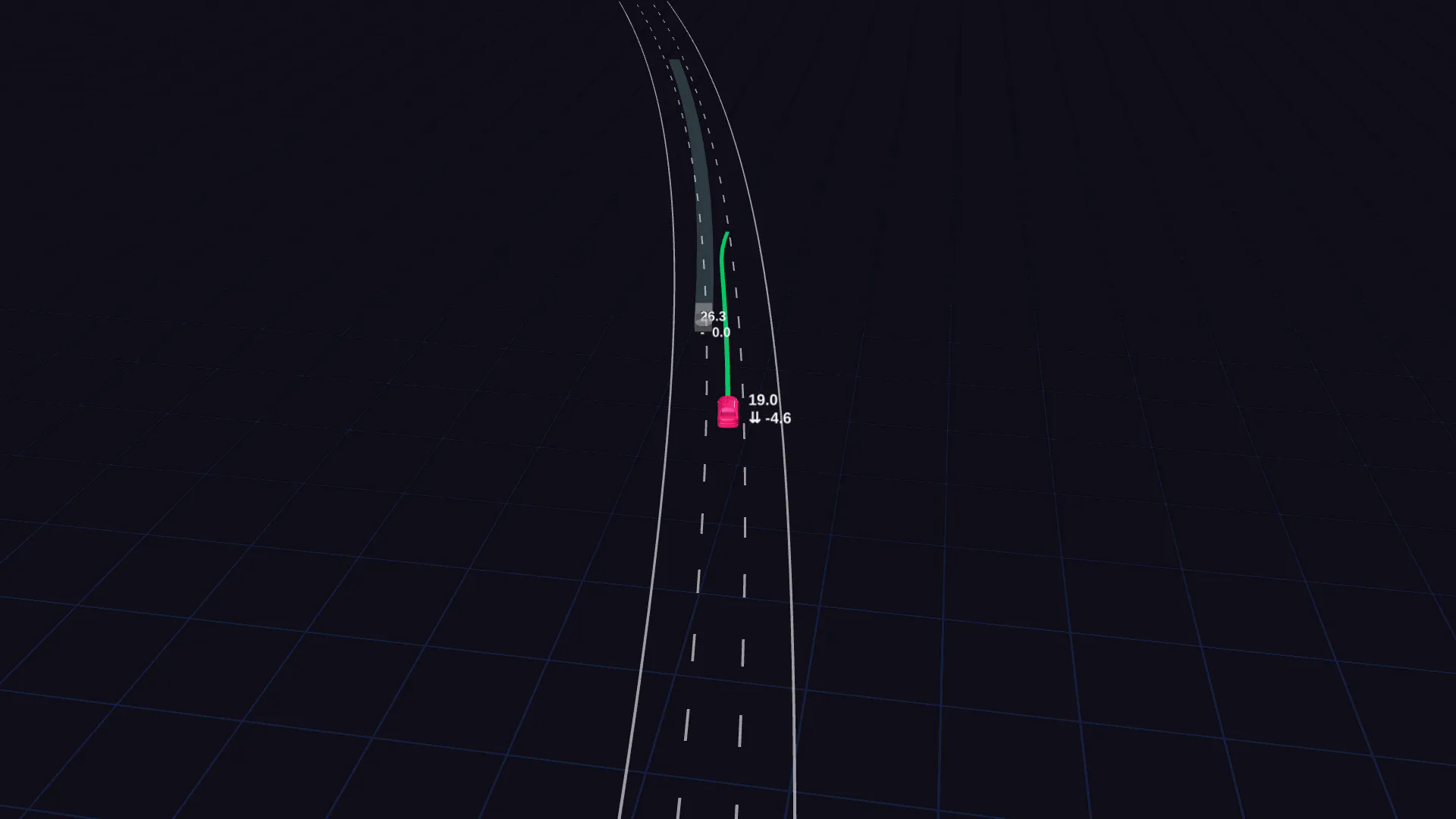}
  \end{subfigure}
  \begin{subfigure}{0.32\textwidth}
    \centering
    \includegraphics[width=\linewidth,trim={24cm 10cm 24cm 10cm},clip]{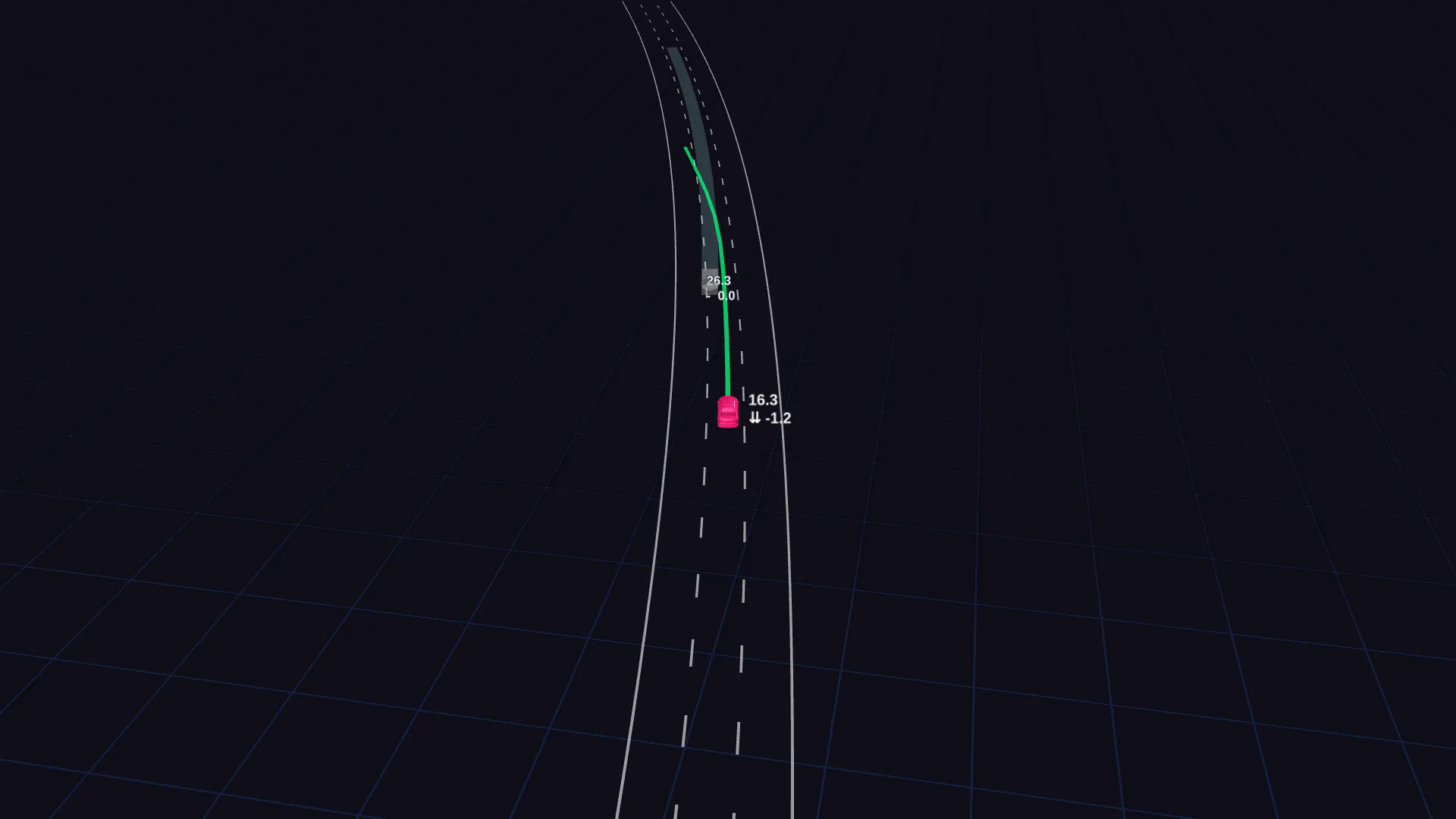}
  \end{subfigure}
  \caption{Our agent reacts to an actor cutting in during a targeted scenario. }
  \label{fig:cut-in-ex}
\end{figure}

\begin{figure}[H]
  \begin{subfigure}{0.32\textwidth}
    \centering
    \includegraphics[width=\linewidth,trim={24cm 10cm 24cm 10cm},clip]{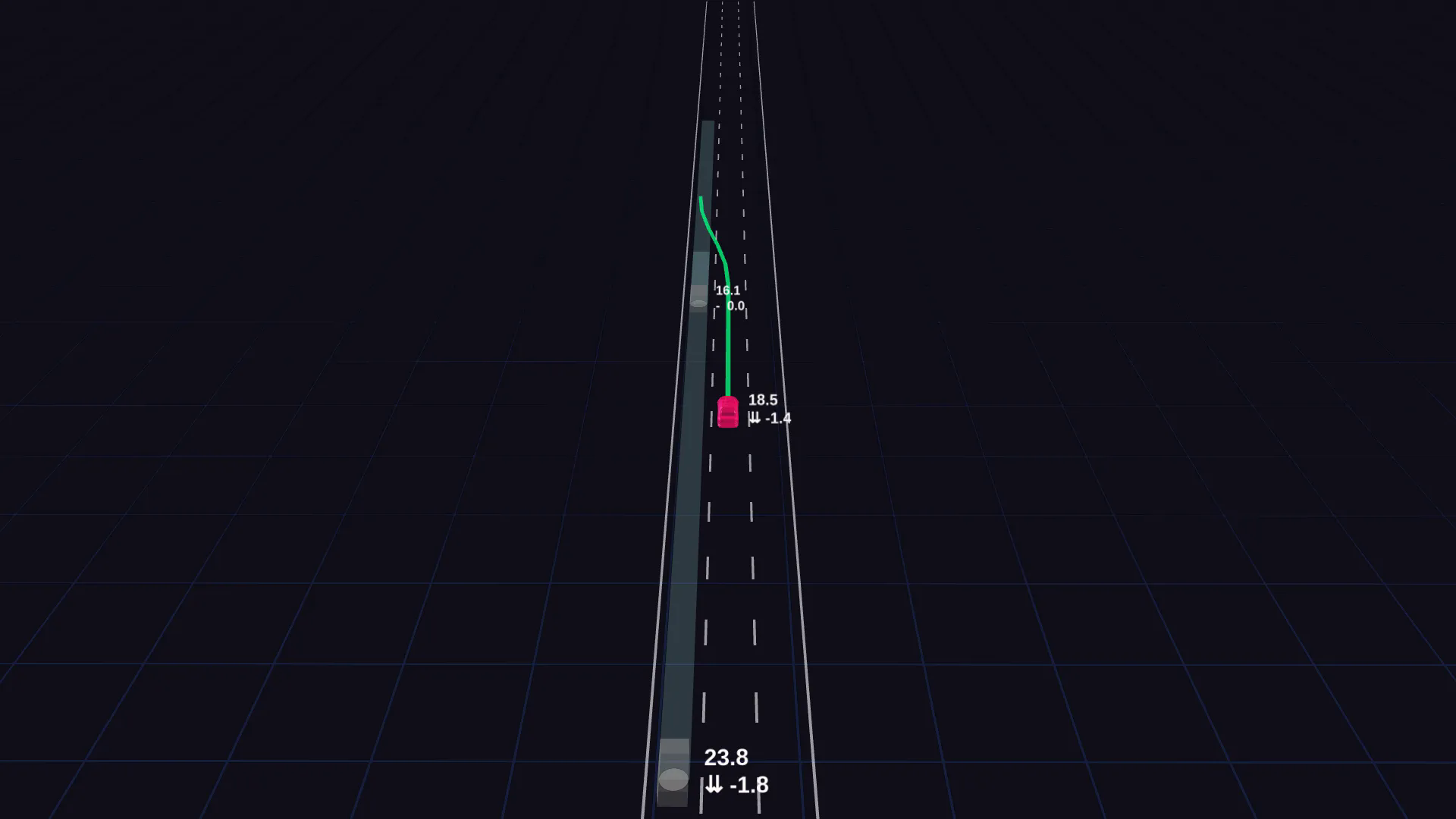}
  \end{subfigure}
  \begin{subfigure}{0.32\textwidth}\
    \centering
    \includegraphics[width=\linewidth,trim={24cm 10cm 24cm 10cm},clip]{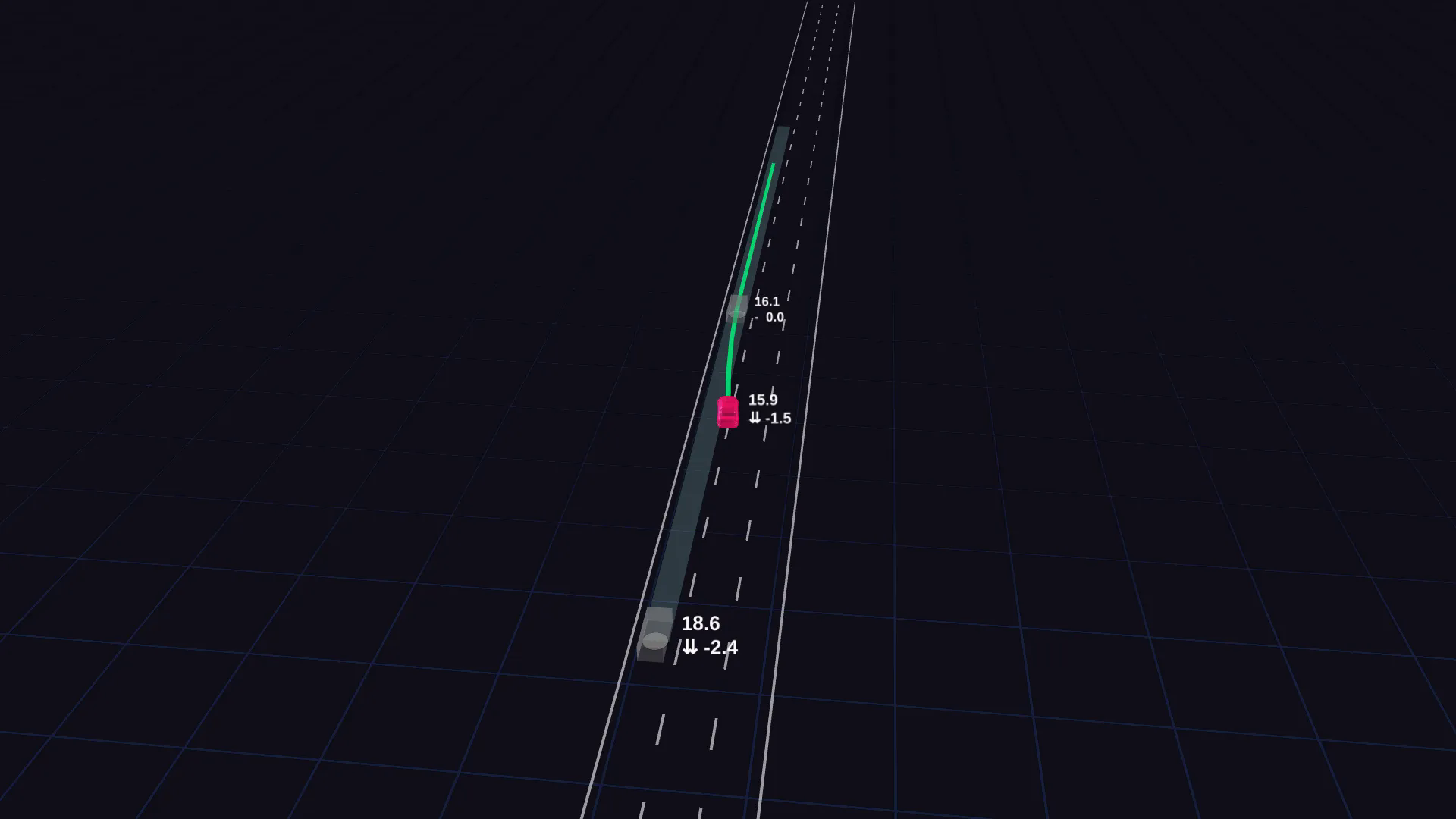}
  \end{subfigure}
  \begin{subfigure}{0.32\textwidth}
    \centering
    \includegraphics[width=\linewidth,trim={24cm 10cm 24cm 10cm},clip]{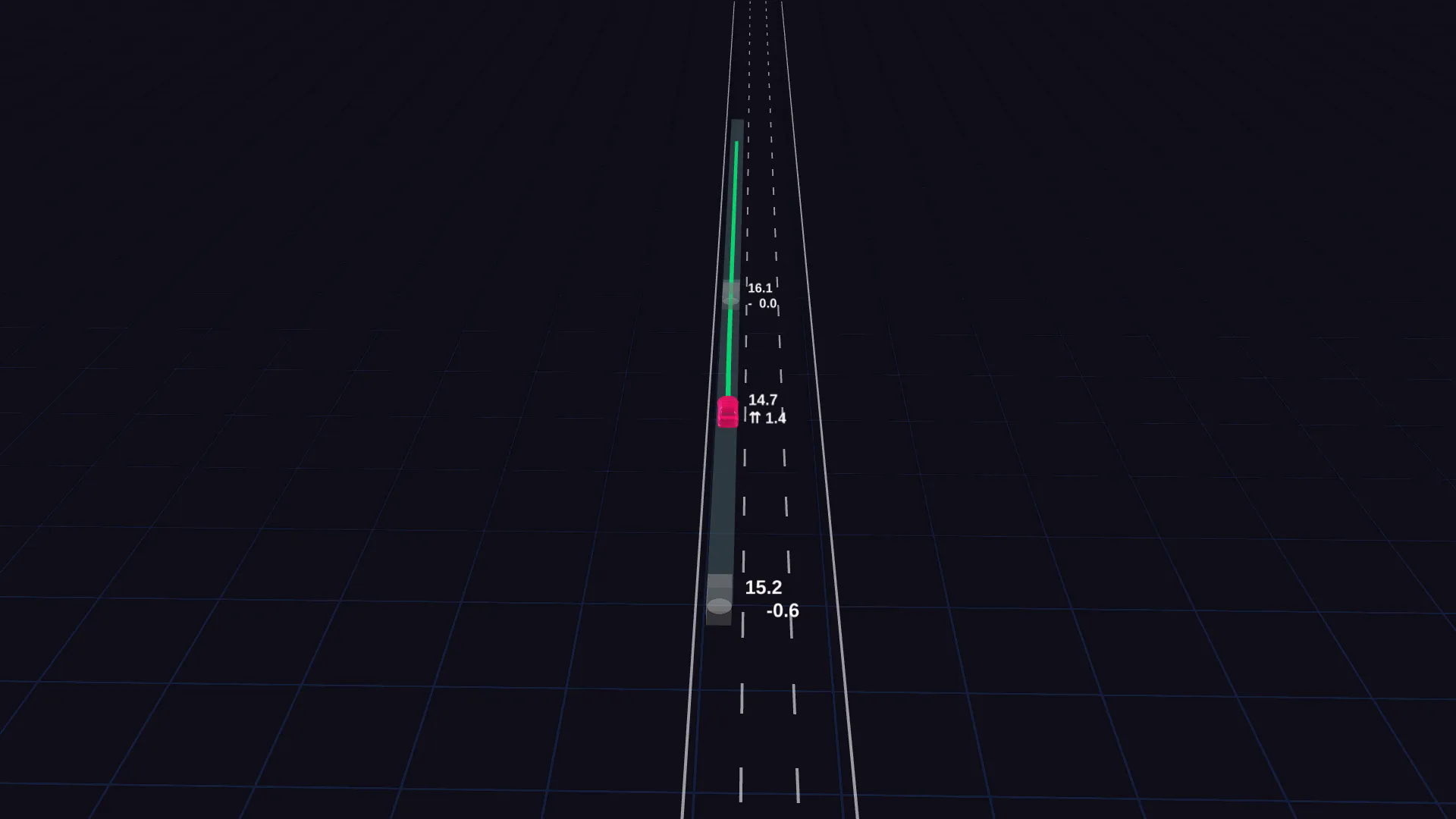}
  \end{subfigure}
  \caption{Our agent slows down in order to lane change between two actors.}
  \label{fig:lane-change-ex}
\end{figure}

\begin{figure}[H]
  \begin{subfigure}{0.32\textwidth}
    \centering
    \includegraphics[width=\linewidth,trim={24cm 10cm 24cm 10cm},clip]{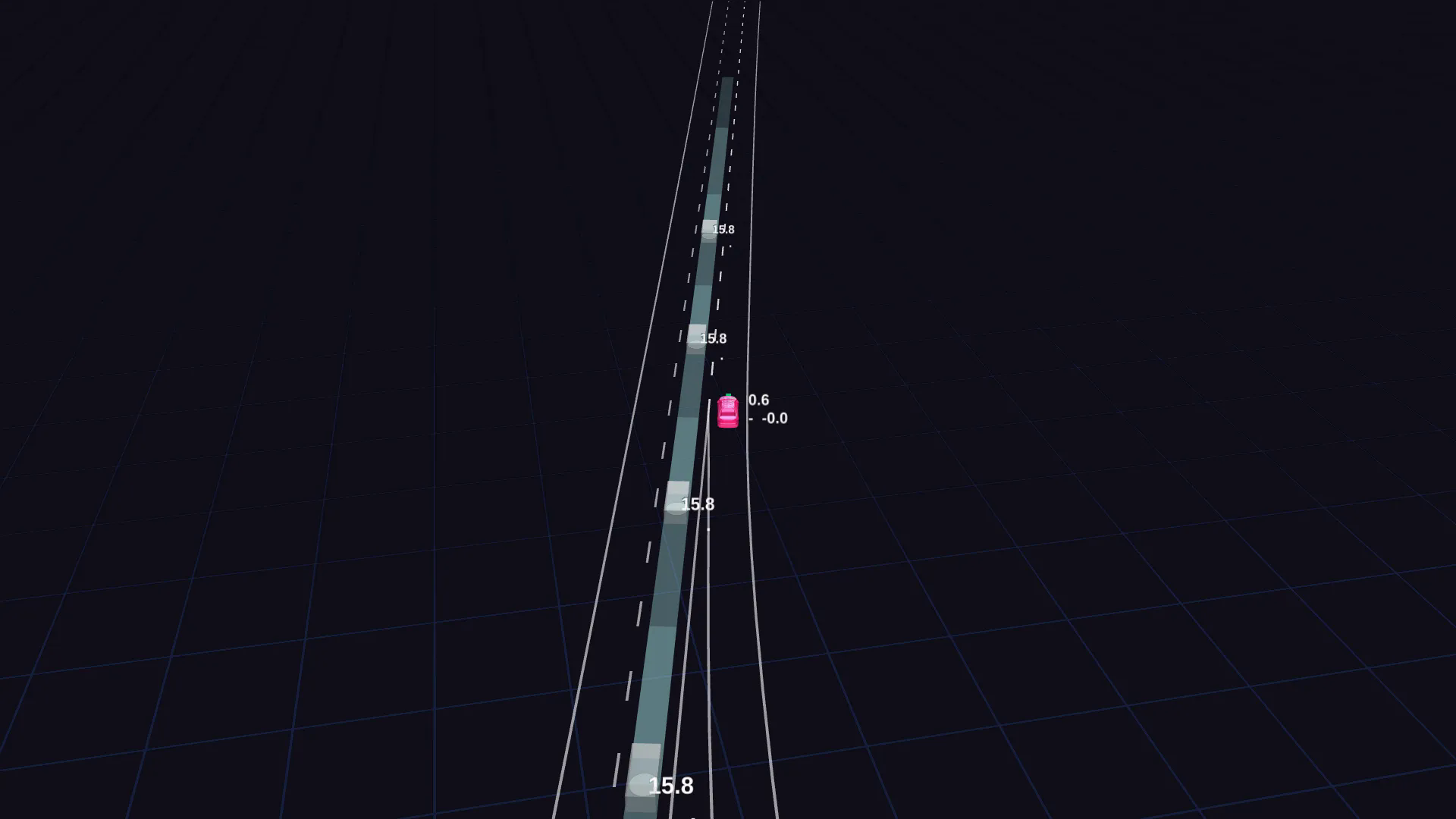}
  \end{subfigure}
  \begin{subfigure}{0.32\textwidth}\
    \centering
    \includegraphics[width=\linewidth,trim={24cm 10cm 24cm 10cm},clip]{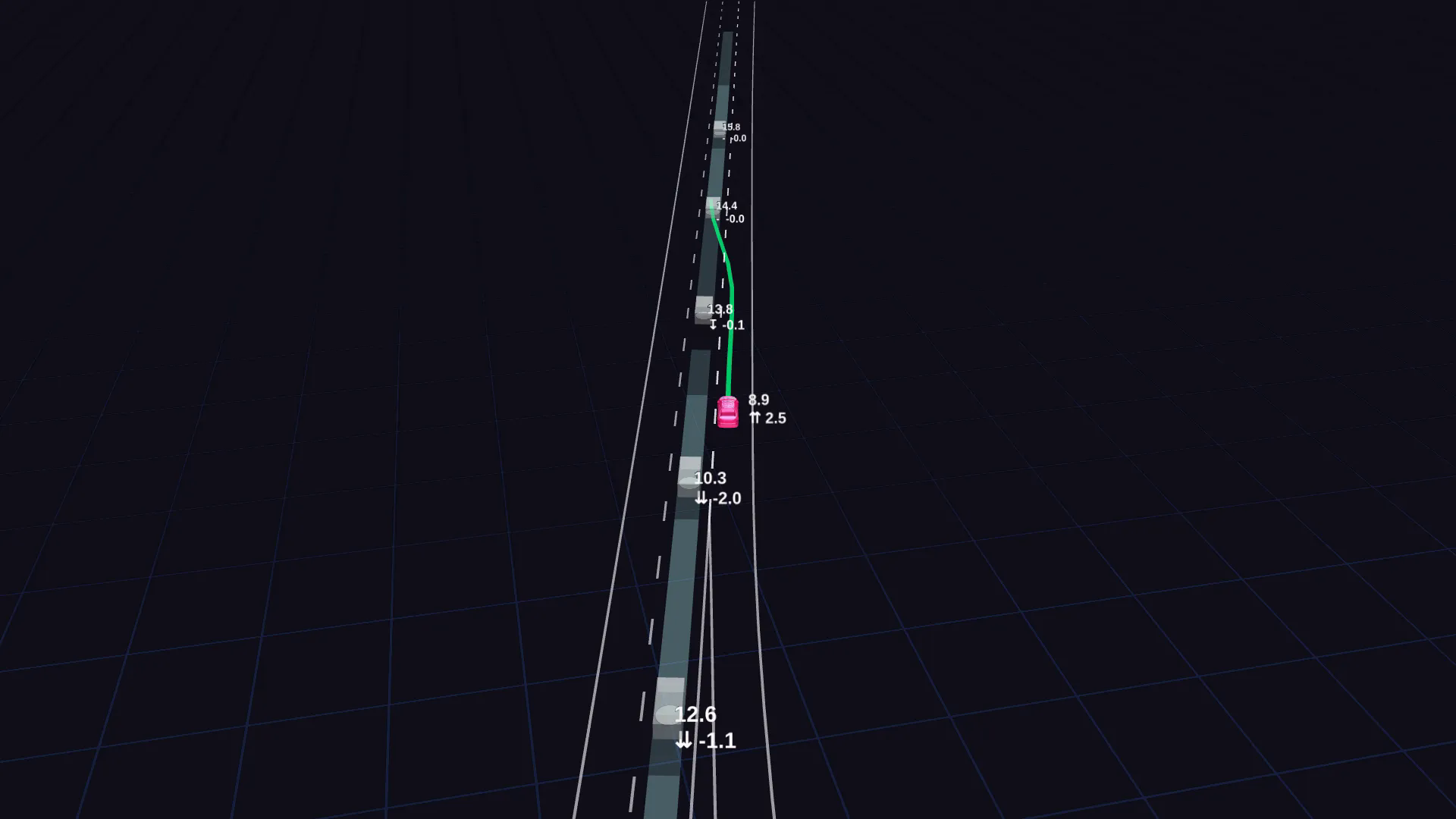}
  \end{subfigure}
  \begin{subfigure}{0.32\textwidth}
    \centering
    \includegraphics[width=\linewidth,trim={24cm 10cm 24cm 10cm},clip]{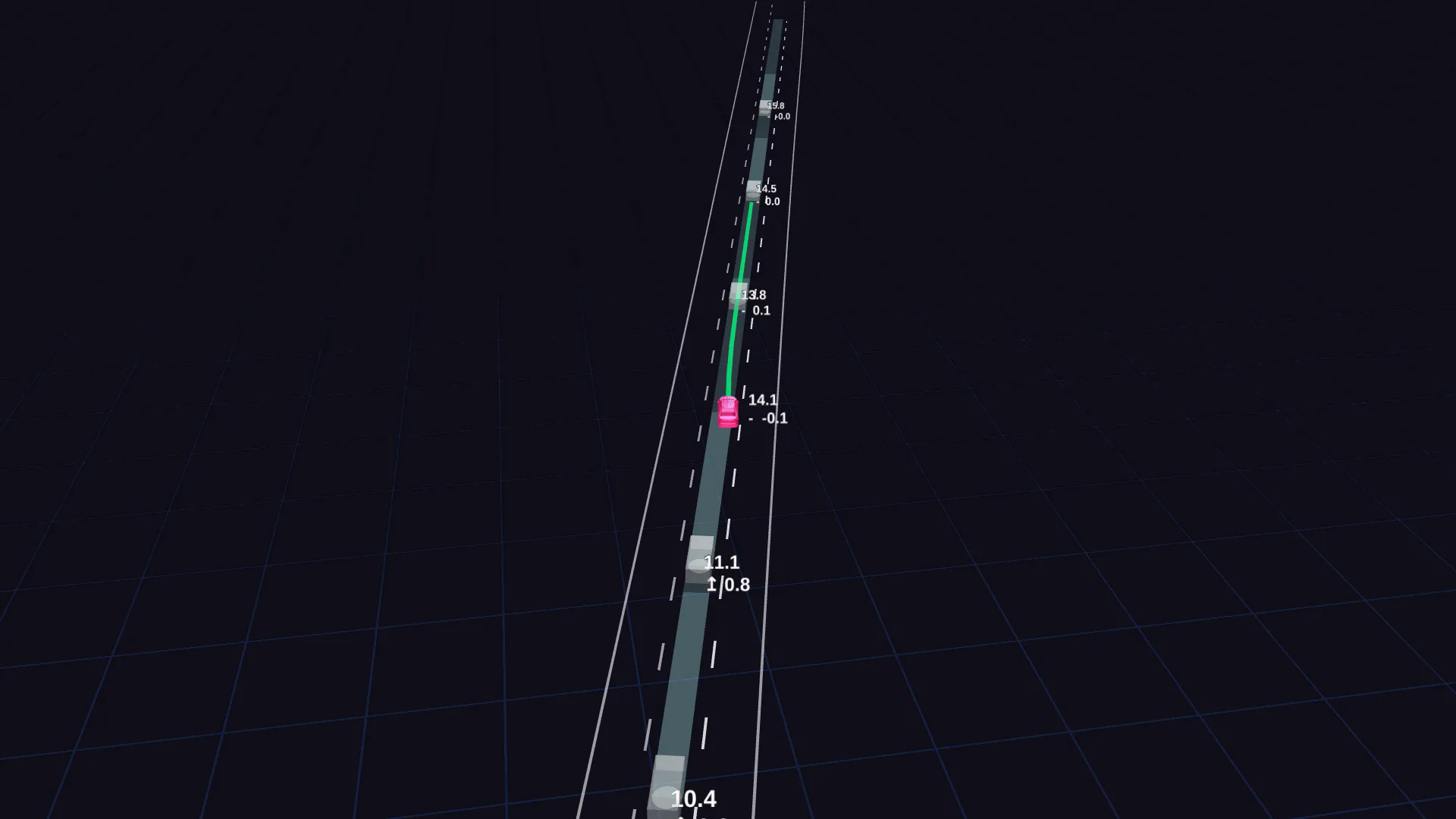}
  \end{subfigure}
  \caption{Our agent is initiated with very low velocity. It has learned that it
    must speed up in order to merge into traffic.}
  \label{fig:merge-ex}
\end{figure}

\begin{figure}[H]
  \begin{subfigure}{0.32\textwidth}
    \centering
    \includegraphics[width=\linewidth,trim={24cm 10cm 24cm 10cm},clip]{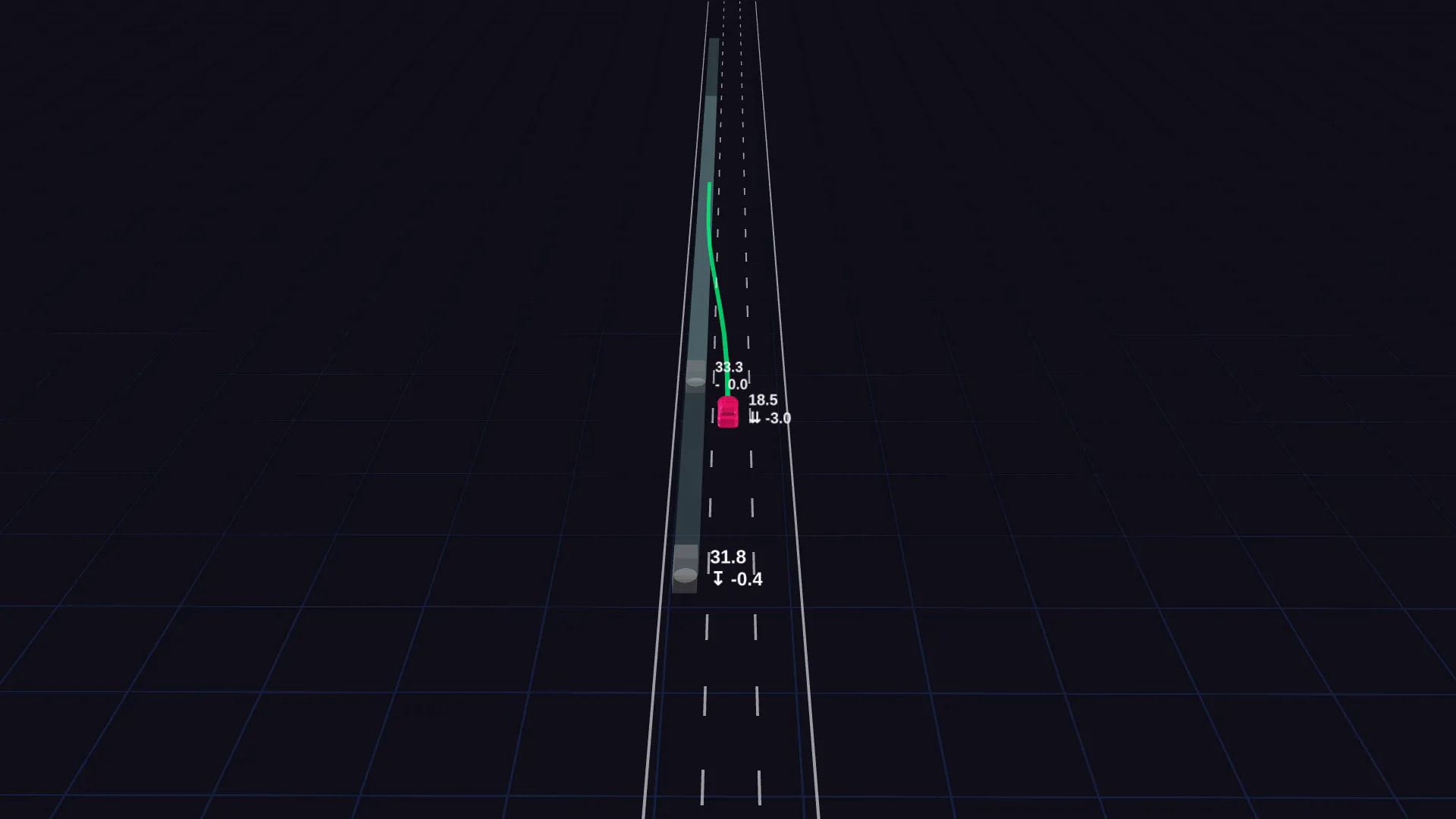}
  \end{subfigure}
  \begin{subfigure}{0.32\textwidth}\
    \centering
    \includegraphics[width=\linewidth,trim={24cm 10cm 24cm 10cm},clip]{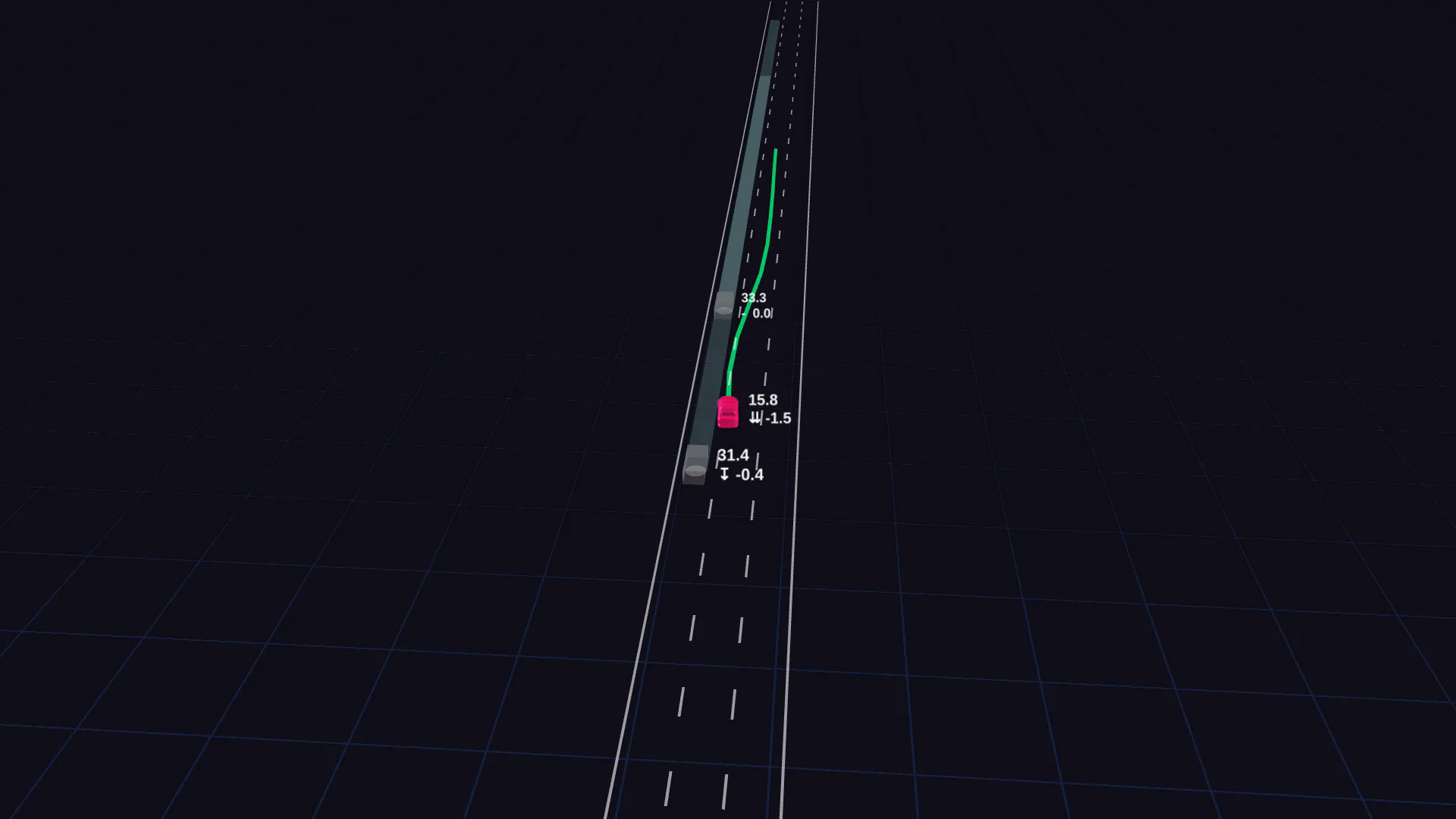}
  \end{subfigure}
  \begin{subfigure}{0.32\textwidth}
    \centering
    \includegraphics[width=\linewidth,trim={24cm 10cm 24cm 10cm},clip]{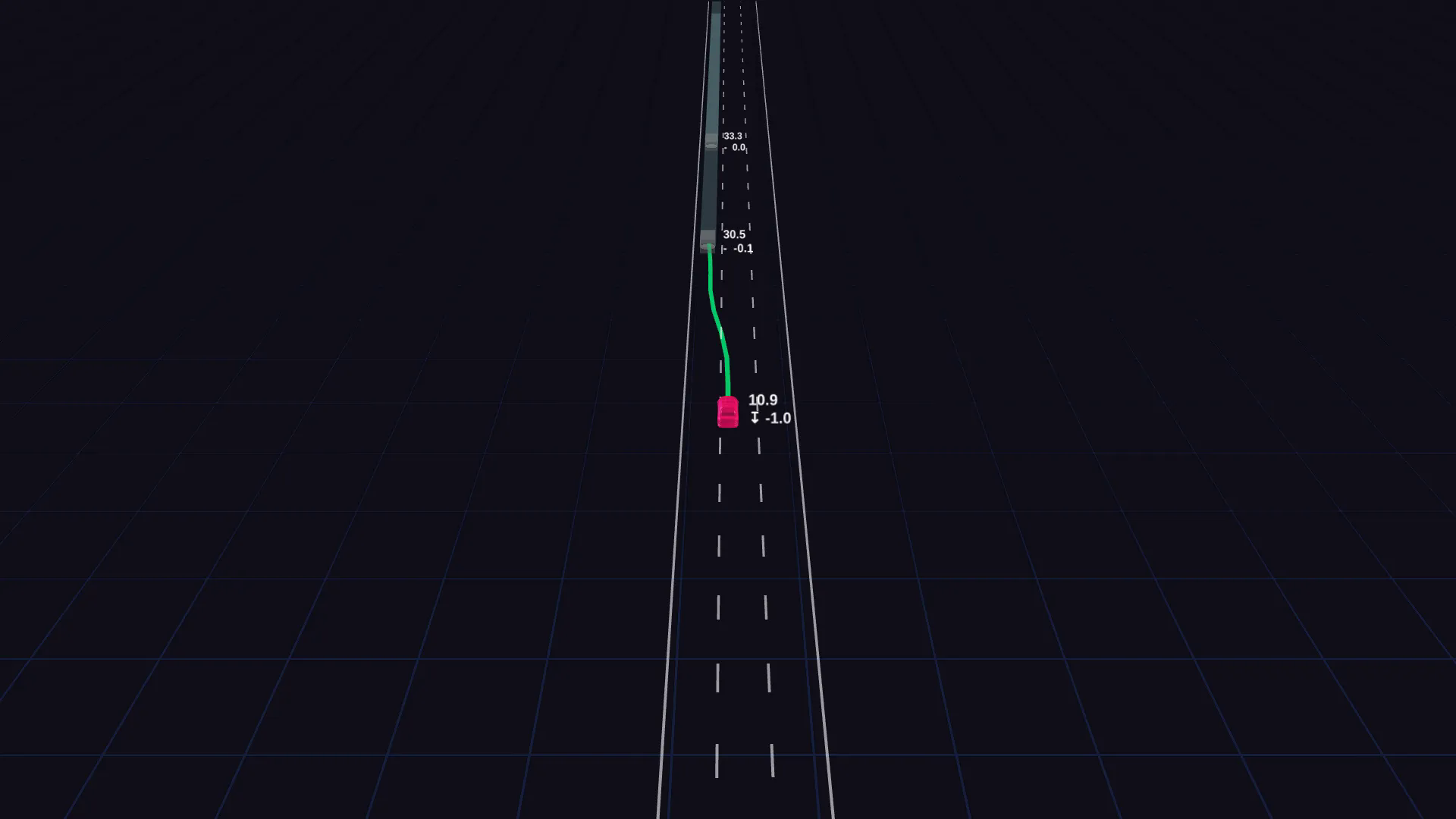}
  \end{subfigure}
  \caption{Here we see a failure case of our model. The agent makes a bad
    decision to initiate a lane change before making a last minute adjustment to
    avoid collision. While collision is avoided, this is still unsafe behavior. }
  \label{fig:lane-change-fail-ex}
\end{figure}

\end{document}